\documentclass[letterpaper]{article} 
\usepackage[preprint]{aaai2027}  
\usepackage{times}  
\usepackage{helvet}  
\usepackage{courier}  
\usepackage[hyphens]{url}  
\usepackage{graphicx} 
\usepackage{natbib}  
\usepackage{caption} 
\usepackage{amsmath}
\usepackage{amssymb}
\usepackage{amsfonts}
\usepackage{bm}
\usepackage{booktabs}
\usepackage{colortbl}
\usepackage{multirow}
\usepackage{amsthm}
\usepackage{algorithm}
\usepackage{algpseudocode}
\usepackage{float} 
\usepackage{microtype} 

\newtheorem{proposition}{Proposition}
\newtheorem{theorem}{Theorem}
\newtheorem{lemma}{Lemma}
\newtheorem{corollary}{Corollary}
\newtheorem{assumption}{Assumption}
\newtheorem{definition}{Definition}
\newtheorem{remark}{Remark}

\newcommand{\Bshared}{B_{\mathrm{shared}}}
\newcommand{\Bdir}{B_{\mathrm{dir}}}
\newcommand{\Bnorm}{B_{\mathrm{norm}}}
\newcommand{\ACG}{\mathrm{ACG}}
\newcommand{\Jm}{J_m}
\newcommand{\Cmat}{C}
\newcommand{\Fmat}{F}
\newcommand{\Rproxy}{\mathcal{R}}
\newcommand{\Lpost}{\mathcal{L}_{\mathrm{post}}}
\newcommand{\Lproxy}{\mathcal{L}_{\mathrm{proxy}}}
\DeclareRobustCommand{\passk}{\textup{pass@}\ensuremath{k}}
\DeclareRobustCommand{\passone}{\textup{pass@}\ensuremath{1}}
\newcommand{\Dm}{\mathcal{D}_m}
\newcommand{\1}{\mathbf{1}}
\newcommand{\real}{\mathbb{R}}
\newcommand{\CDRFT}{CD-RFT}

\newcommand{\addon}[1]{\hspace{0.6em}+\,#1}
\newcommand{\ourrow}{\rowcolor{black!8}}

\title{Control-Diverse Reinforcement Fine-Tuning: Decoupling the Shared Control Bottleneck of RL Post-Training}
\author{
    Binwen Tan$^{1}$\equalcontrib,
    Jingchao Wang$^{3}$\equalcontrib,
    Dengzhe Hou$^{1,2}$,
    Lingyu Jiang$^{1}$,\\
    Zeyuan Wu$^{4}$,
    Yunhan Shen$^{1}$,
    Fangzhou Lin$^{5,6}$,
    Kazunori Yamada$^{1,2}$,
    Atsushi Koike$^{1}$
}
\affiliations{
    $^{1}$Graduate School of Information Sciences, Tohoku University\\
    $^{2}$Unprecedented-scale Data Analytics Center, Tohoku University\\
    $^{3}$School of Computer Science, Peking University\\
    $^{4}$School of Science, Tohoku University\\
    $^{5}$Texas A\&M University\\
    $^{6}$Worcester Polytechnic Institute
}

\makeatletter
\def\cd@wtoc#1#2#3{\addtocontents{#1}{\protect\contentsline{#2}{#3}{\thepage}{}}}
\newcommand{\CDenableToc}{%
  \def\addcontentsline##1##2##3{%
    \def\cd@lvl{##2}\def\cd@sec{section}\def\cd@sub{subsection}%
    \ifx\cd@lvl\cd@sec \cd@wtoc{##1}{##2}{##3}\else
    \ifx\cd@lvl\cd@sub \cd@wtoc{##1}{##2}{##3}\fi\fi}}
\makeatother

\begin{document}
\maketitle

\begin{abstract}
Reinforcement learning post-training unlocks complex reasoning in large language models. Yet benchmark scores reveal only whether a model improved, not what changed inside it, nor how it splits a finite capability across competing tasks. A representative line of work in mechanistic interpretability attributes the success of reinforcement-learning fine-tuning to stronger and more diverse circuit activation. Complementing this view, we separate activation from control: an activated circuit need not control the post-training reward gain. Adapting Metabolic Control Analysis, we define the Post-training Control Coefficient to measure component control over the reward gain and arrange these coefficients by task family into a control matrix, paired with an activation-magnitude matrix. We call cross-task control concentration the Shared Control Bottleneck and the difference between activation and control concentration the Activation--Control Gap. We show that highly shared activations can coexist with task-specific control, while a small gap indicates that control concentrates along a shared direction and loses task specificity. To reduce this concentration, we regularize the post-training loss with the Shared Control Bottleneck and propose Control-Diverse Reinforcement Fine-Tuning (\CDRFT{}). The exact regularizer gradient requires second-order automatic differentiation incompatible with flash attention, so we derive a first-order proxy with worst-case overhead below eight percent. On Qwen2.5-7B, \CDRFT{} achieves the largest control decoupling and improves multi-task capability over its matched GRPO recipe across all three domains. The no-KL variant leads on \passone{}, and the KL-penalized variant leads on the large-$k$ \passk{} coverage that KL otherwise degrades. Together, these results show that the Shared Control Bottleneck serves as a mechanistic diagnostic and a training regularizer, and that both control decoupling and capability gains transfer to Llama-3.2-3B.
\end{abstract}

\section{Introduction}\label{sec:intro}

\begin{figure}[!t]
\centering
\includegraphics[trim=49.7pt 23.8pt 51.5pt 27.4pt, clip, width=\linewidth]{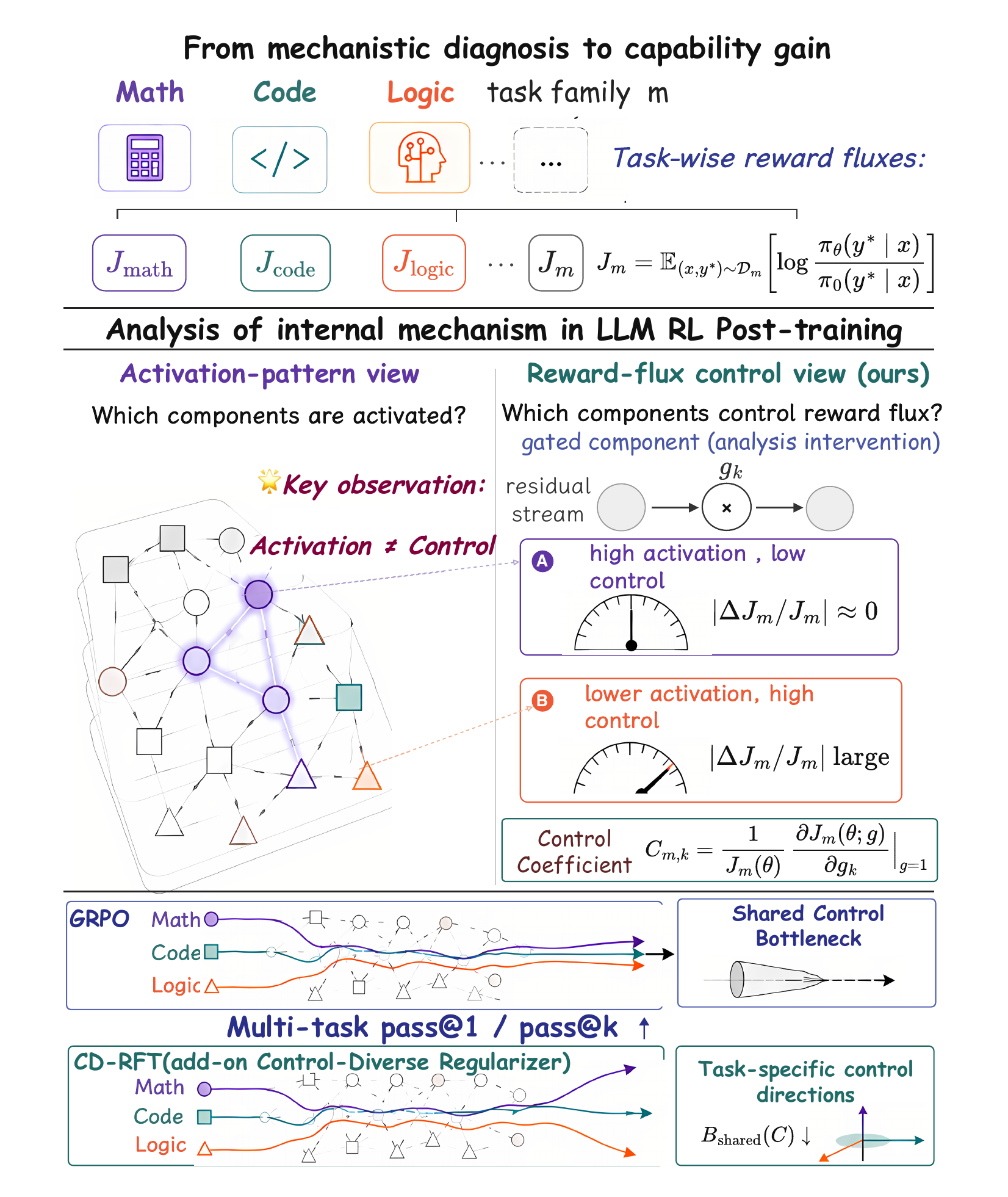}
\caption{The control view at a glance. \textbf{Top:} task-wise reward flux $\Jm$ (Eq.~\eqref{eq:flux}). \textbf{Middle:} attribution patching \citep{syed2024attribution} through a differentiable gate on each component yields the control coefficient $C_{m,k}$ (Eq.~\eqref{eq:coeff}); A and B show that activation does not imply control. \textbf{Bottom:} \CDRFT{} keeps control task-specific across families, lowering $\Bshared(\Cmat)$ and raising \passone{}/\passk{}.}
\label{fig:concept}
\end{figure}

The complex reasoning of large language models is unlocked primarily during post-training, where reinforcement learning maximizes a reward while holding the policy near its base model \citep{ouyang2022training}. As post-training moves from single-task to multi-task settings, a central question emerges: what does it actually change inside the model, and how does it distribute a finite capability across competing tasks?

Benchmark scores answer only whether a model has improved, not where the change occurs. Mechanistic interpretability quantifies, through attribution patching, the causal role of each component \citep{syed2024attribution}; along this line, a representative work attributes the success of reinforcement-learning fine-tuning to stronger and more diverse circuit activation \citep{zhang2025reinforcement}. We complement this view: activation tells us which circuits fire, but not which ones control the post-training reward gain, and that account is established on activation alone, within a single domain (mathematical reasoning).

To separate which components fire from which ones control the gain, we draw on Metabolic Control Analysis \citep{kacser1973flux,heinrich1974linear,fell1992metabolic}, the systems-biology framework for asking which local steps control a network's overall flux, and define the reward flux as the mean log-likelihood ratio margin by which the policy raises a reference target above the base model; a differentiable gate on each component then yields the Post-training Control Coefficient, read out by a single backward pass, measuring which components control the gain. Task-wise coefficients and component activation magnitudes form the control and activation matrices, respectively. We call cross-task control concentration the Shared Control Bottleneck and the difference between activation and control concentration the Activation--Control Gap. Highly shared activations can retain task-specific control; a small gap indicates that control concentrates along a shared direction and loses task specificity.

We turn this diagnosis into \CDRFT{}, which regularizes the post-training loss with the Shared Control Bottleneck.\footnote{Code is available at \url{https://github.com/tttbw/cd-rft}.} The exact regularizer gradient requires second-order differentiation incompatible with flash attention, so we derive a first-order proxy requiring one backward pass and less than eight percent overhead. On Qwen2.5-7B, \CDRFT{} attains the largest control decoupling. Over its matched GRPO recipe, the no-KL variant improves \passone{} and the KL-penalized variant improves large-$k$ \passk{} coverage. Beyond Qwen2.5-7B, the same pattern of control decoupling and capability gain carries over to Llama-3.2-3B.

Our contributions are as follows.
\begin{itemize}

\item We define the Post-training Control Coefficient, which separates activation contribution from control responsibility.

\item We characterize excessive cross-task control sharing using the Shared Control Bottleneck and the Activation--Control Gap, and derive an optimizable first-order proxy for the bottleneck.

\item We propose \CDRFT{}, which connects mechanistic diagnosis to training by decoupling the control structure; reducing the diagnosed bottleneck improves multi-task \passone{} and \passk{} on Qwen2.5-7B, with the same pattern on Llama-3.2-3B.

\end{itemize}

\section{Related Work}\label{sec:related}

\paragraph{Reinforcement learning post-training paradigm for LLM.}
Reference-constrained post-training (reinforcement learning from human feedback \citep{ouyang2022training} and direct preference optimization \citep{rafailov2023direct}) shares the log-ratio structure on which our reward flux is based. Reinforcement learning with verifiable rewards has become the dominant paradigm for reasoning: group relative policy optimization removes the value network \citep{shao2024deepseekmath}, pure reinforcement learning elicits reasoning behaviours \citep{guo2025deepseekr1}, whose gains over the base model, especially at large-$k$ $\passk$, are recently questioned \citep{yue2025rl}, and later work extends it along training stability \citep{yu2025dapo} and data coverage \citep{luo2025deepscaler,cheng2025guru}. Closest to our setting, multi-task GRPO jointly trains a single policy across several task families, re-weighting tasks to balance worst-task performance \citep{ramesh2026mtgrpo}. Whether acting on the reward, the data, or the multi-task objective, however, such methods stay at the external-signal level and do not characterize how a finite capability is distributed inside the model, or how concentrated that distribution is; we instead lower the Shared Control Bottleneck directly in the control space.

\paragraph{Mechanistic interpretability and causal attribution of fine-tuning.}
Mechanistic interpretability quantifies the causal role of components by attribution patching \citep{syed2024attribution,nanda2023attribution}, an idea inherited from the circuit framework \citep{elhage2021framework} and circuit discovery \citep{conmy2023acdc,wang2022ioi,kramar2024atp} and made more faithful by integrated gradients \citep{sundararajan2017axiomatic,hanna2024faith}; on fine-tuning, one line shows that it amplifies or reuses existing mechanisms rather than reconstructing circuits \citep{prakash2024finetuning,jain2024mechanistically}. Most relevant to us, \citet{zhang2025reinforcement} attribute the success of reinforcement-learning fine-tuning to increased activation intensity and diversity, measured within a single domain (mathematical reasoning) by three statistics of the same EAP edge-magnitude tensor: its mean (activation intensity), its entropy (information complexity), and its kurtosis (distribution kurtosis). All three describe how large and how spread the attributions are, not which components the reward gain depends on, a distinction Lemma~\ref{lemma:axes} makes precise: a circuit being activated does not entail that it controls the gain. That line of work is moreover purely diagnostic. We therefore redirect attribution to the Post-training Control Coefficient over the reward flux, characterize the over-sharing of control for the first time through the cross-task Shared Control Bottleneck, and close the loop by writing that quantity into the training objective and measuring the resulting multi-task gain.

\section{The Post-Training Control Coefficient}\label{sec:prelim}

\subsection{Post-Training Reward Flux}\label{sec:flux}

To separate activation from control into checkable quantities, we take as our basic object the log-likelihood-ratio margin of the policy relative to the reference model on a reference target. Our main experiments use reinforcement learning with verifiable rewards \citep[RLVR;][]{lambert2024tulu3}, in which the reference target $y^\star$ is a verified-correct completion; for a task family $m$ with samples $(x,y^\star)\sim\Dm$, we define the reward flux
\begin{equation}\label{eq:flux}
\Jm(\theta)=\mathbb{E}_{(x,y^\star)\sim\Dm}\!\left[\log\pi_\theta(y^\star\mid x)-\log\pi_0(y^\star\mid x)\right].
\end{equation}
The integrand $\log(\pi_\theta/\pi_0)$ in Eq.~\eqref{eq:flux} is the common structure of the family of reference-constrained post-training objectives: RLVR takes verifiable correctness as the reward and raises this log-likelihood ratio on correct targets, DPO writes the implicit reward as $\beta\log(\pi_\theta/\pi_0)$ \citep{rafailov2023direct}, and RLHF targets a trade-off between the reward and $\mathrm{KL}(\pi_\theta\,\|\,\pi_0)$ \citep{ouyang2022training}. Hence $\Jm$ is not an artifact constructed for the analysis but this shared log-ratio structure evaluated at the reference target.

\subsection{Metabolic Control Analysis}\label{sec:mca}

Metabolic Control Analysis (MCA) originates in systems biology \citep{kacser1973flux,heinrich1974linear} and studies the share of control that a single local step exerts over the overall flux of a network. Its central object, the flux control coefficient, is the scaled sensitivity (log--log derivative) of the overall flux $J$ to a local rate $v_i$,
\begin{equation}
C^{J}_{i}=\frac{\partial\ln J}{\partial\ln v_i}=\frac{v_i}{J}\,\frac{\partial J}{\partial v_i},
\end{equation}
and its summation theorem states that in a closed conservative network the control coefficients sum to one, so control is systematically distributed rather than monopolized \citep{fell1992metabolic}. We transplant this ``how local rates control the overall flux'' viewpoint to post-training, taking the reward flux as the overall flux and the component gate as the local rate. The residual structure breaks the required homogeneity, so $\sum_k C_{m,k}\neq 1$ here, as Appendix Proposition~\ref{app:propB1} shows; the distribution of control does not rely on this conservation law, however, and Theorem~\ref{thm:nosingle} proves independently that residual-stream superposition alone keeps the control support from collapsing to a single gate. We therefore borrow only the ``control share'' viewpoint of MCA.

\subsection{Differentiable Gating and Attribution Patching}\label{sec:gating}

On the residual update of each layer, we multiply a set of internal components each by a differentiable scalar gate $g_k$,
\begin{equation}\label{eq:gating}
h^{\ell+1}=h^\ell+\sum_{k:\ell(k)=\ell}g_k\,f_k(h^\ell;\theta),
\end{equation}
where $f_k$ is the transformation of component $k$ and $\mathrm{out}_k$ is its output written into the residual stream; when the gates take their nominal value $g\equiv\1$ the network recovers the original model under Appendix Assumption~\ref{app:asmA1}, so that $g_k$ is a continuous ablation intervention at the component level ($0$ for full ablation, $1$ for intact), and we denote the gated reward flux by $\Jm(\theta;g)$. The gate granularity is a free choice: an attention head, an MLP neuron, or a circuit edge. We use the sublayer granularity, the coarsest choice under which the spectral-proxy optimization stays feasible for full 7B fine-tuning; finer gate sets (head, neuron, or edge) are left to future work.

To measure the causal effect of such an intervention on a target metric, the standard tool in mechanistic interpretability is \emph{attribution patching} \citep{syed2024attribution,nanda2023attribution}, which linearizes an activation replacement by a first-order Taylor expansion (Appendix Eq.~\eqref{app:eq:ap}), so a single backward pass gives the attributions of all components at once. We take the gate as the intervention variable and the target log-likelihood as the metric, and read out the control coefficient next.

\subsection{The Coefficient and Its Expansion}\label{sec:coeff}

Combining the flux control coefficient with the gating above, we define the post-training control coefficient as the scaled gate-sensitivity of the target log-likelihood $\ell_m(\theta):=\mathbb E_{\Dm}[\log\pi_\theta(y^\star\mid x)]$, the policy term of the flux,
\begin{equation}\label{eq:coeff}
C_{m,k}=\frac{1}{\ell_m(\theta)}\,\frac{\partial \ell_m(\theta;g)}{\partial g_k}\bigg|_{g=\1}.
\end{equation}
The reference term $\log\pi_0$ is gate-independent, so $\partial\ell_m/\partial g_k=\partial\Jm/\partial g_k$: the coefficient reports the gate-sensitivity of the reward flux, now normalized by the stable $\ell_m$ ($|\ell_m|\ge\epsilon_0$, Appendix Assumption~\ref{app:asmA2}) rather than by the sign-indefinite, possibly-vanishing $\Jm$ (Appendix Remark~\ref{app:remA1}). Arranging the control coefficients into the control matrix $\Cmat=[C_{m,k}]\in\real^{M\times K}$ with $M=M_{\mathrm{fam}}$ rows, one per task family: each row is the control vector of family $m$, averaged over the family's $n$ probe samples, as detailed in Appendix~\ref{sec:appendix}.

The explanatory power of the control coefficient comes from its first-order dominance over the relative change of the reward flux under small gate perturbations.

\begin{proposition}[first-order expansion of control]\label{prop:expansion}
Under Appendix Assumptions~\ref{app:asmA0}--\ref{app:asmA2}, and writing the gate log-perturbation $u=\log g$, the relative change of the target log-likelihood of family $m$ satisfies
\begin{equation}\label{eq:expansion}
\frac{\ell_m(\theta;e^{u})-\ell_m(\theta)}{\ell_m(\theta)}=\sum_{k=1}^{K}C_{m,k}\,u_k+O\!\left(\lVert u\rVert_2^2\right),\qquad u\to\mathbf 0.
\end{equation}
\end{proposition}
This is a first-order Taylor expansion about $u=\mathbf 0$ divided by $\ell_m(\theta)\neq 0$, proven in Appendix Proposition~\ref{app:propA1}.

Equation~\eqref{eq:expansion} makes the control coefficient the first-order dominant coefficient of the relative change in gate log-space, measuring the control of component $k$ over the gain of family $m$. This differs from the activation magnitude tracked by prior analyses \citep{zhang2025reinforcement}, the average usage of component $k$ on family $m$,
\begin{equation}
F_{m,k}=\mathbb{E}_{\Dm}\lVert\mathrm{out}_k(x;\theta)\rVert_2,\qquad F=[F_{m,k}]\in\real^{M\times K}.
\end{equation}
Because the reward flux is read out along a fixed direction $w_m$ defined in Appendix~\ref{sec:appendix}, the first-order effect of a component depends on the part of its write that survives propagation to $w_m$, not on its norm; large usage and strong control thus decouple.

\begin{lemma}[the activation axis and the control axis are not interchangeable]\label{lemma:axes}
For the gating of Eq.~\eqref{eq:gating} there are networks and task families in which $C_{m,\cdot}$ is not a scalar multiple of $F_{m,\cdot}$.
\end{lemma}
Appendix Lemma~\ref{app:lemA1} constructs one.

Lemma~\ref{lemma:axes} has two consequences. First, ``post-training enhances activation diversity'' does not entail ``post-training improves the control structure'': the two axes are not interchangeable. Second, optimizing on the activation axis changes the objective and cannot reduce control sharing (Section~\ref{sec:sensitivity}). We therefore treat the control matrix as a first-class object.

\section{The Shared Control Bottleneck}\label{sec:motivation}

Distinguishing control from activation is what makes control-aware training possible (Section~\ref{sec:method}). This section builds that measurement: the Activation--Control Gap between the sharing of control and of activation.

\subsection{Cross-Task Concentration}\label{sec:concentration}

For a matrix $X\in\{\Fmat,\Cmat\}$, its family Gram matrix $G_X=XX^\top\in\real^{M\times M}$ is symmetric positive semidefinite, and we characterize the cross-task concentration by its normalized largest eigenvalue,
\begin{equation}\label{eq:concentration}
\Bshared(X)=\frac{\lambda_{\max}(G_X)}{\operatorname{tr}(G_X)}\in\Big[\tfrac1M,1\Big],
\end{equation}
with $X\neq 0$ so that the denominator is nonzero. This quantity measures how much of the energy of the rows (one per task family) falls along a single shared direction.

\begin{proposition}[bounds, saturation, and invariance of $\Bshared$]\label{prop:bounds}
Let the eigenvalues of $G_X$ be $\lambda_1\ge\dots\ge\lambda_M\ge0$. Then: (i) $1/M\le\Bshared(X)\le 1$; (ii) the upper bound $\Bshared=1$ holds if and only if $\operatorname{rank}(X)=1$, i.e.\ the row vectors are collinear; (iii) the lower bound $\Bshared=1/M$ holds if and only if all $\lambda_i$ are equal, i.e.\ the row vectors are pairwise orthogonal and of equal norm; (iv) it is invariant under magnitude scaling $X\mapsto cX$ ($c\neq0$) and under orthogonal reparameterization of the gate coordinates $X\mapsto XQ$ ($QQ^\top=I_K$).
\end{proposition}
Appendix Proposition~\ref{app:propB2} proves these.

Property (iv) therefore makes $\Bshared(\Cmat)$ comparable across methods and training steps at a fixed granularity, the prerequisite for a cross-method mechanistic comparison. $\Bshared(\Cmat)$ measures direction sharing and the distribution of control energy jointly.

\subsection{The Activation--Control Gap}\label{sec:acg}

Applying the concentration to the activation axis and the control axis separately, their difference characterizes the degree to which control is decoupled from activation,
\begin{equation}\label{eq:acg}
\ACG=\Bshared(\Fmat)-\Bshared(\Cmat).
\end{equation}
That multi-task activations are shared across tasks is an expected property, with $\Bshared(\Fmat)$ near its upper bound (measured $\approx 99.6$), and as a ceiling it is not a target.

\begin{proposition}[characterization of the gap]\label{prop:pathology}
Under the activational ceiling $\Bshared(\Fmat)\to 1$, the gap satisfies $\ACG\in[0,\,1-1/M]$, and $\ACG\to 0$ if and only if $\Bshared(\Cmat)\to\Bshared(\Fmat)$, in which case control is as concentrated as activation; $\ACG$ is large if and only if $\Bshared(\Cmat)$ lies far below the ceiling, corresponding to families sharing activations while control stays task-specific.
\end{proposition}
Appendix Proposition~\ref{app:propB3} proves this.

Hence the target is not to suppress $\Bshared(\Fmat)$ but, at a fixed ceiling, to lower $\Bshared(\Cmat)$ and thereby enlarge the gap in Eq.~\eqref{eq:acg}, which is precisely the objective of the method in Section~\ref{sec:objective}. Empirically, reinforcement learning already lowers $\Bshared(\Cmat)$ below the base, so our claim is controlled-variable: on a matched recipe, our method lowers $\Bshared(\Cmat)$ the most, as Appendix Remark~\ref{app:remB1} discusses.

\subsection{Distributedness of Control}\label{sec:distributed}

Taking the overall shape of the control matrix, rather than a single-gate attribution, as the object of analysis is justified by the fact that control is systematically distributed over the residual network. This distributedness does not rely on the summation theorem of Section~\ref{sec:mca} (which fails here, Appendix Proposition~\ref{app:propB1}) but stems from the more fundamental structure of residual-stream superposition.

\begin{theorem}[no single-gate monopoly under residual superposition]\label{thm:nosingle}
Suppose the reward flux is read out by projecting the final residual onto a fixed readout direction $w_m$. Then, for generic network weights, $\partial_{g_k}\Jm|_{g=\mathbf 1}\neq0$ for every $k$: the control support is not a proper subset of $\{1,\dots,K\}$.
\end{theorem}
Intuitively, residual superposition mixes the write of every component into the readout, so a gate with no control would need its write to cancel exactly along $w_m$, a measure-zero coincidence (the propagation identity and the generic-weight argument are in Appendix Theorem~\ref{app:thmB1}). The control support is therefore generically not a single gate, so a single-gate scalar cannot characterize it and the right object is the cross-task concentration $\Bshared(\Cmat)$ of the full matrix (Figure~\ref{fig:proxydist}a). Concentration also implies a deployment-level fragility geometry, larger $\Bshared(\Cmat)$ giving larger perturbation variance, as Appendix Proposition~\ref{app:propB4} shows.

\section{Control-Diverse Regularizer }\label{sec:method}

\subsection{Objective}\label{sec:objective}

On top of the backbone post-training loss $\Lpost$, we add a single control-diverse regularizer,
\begin{equation}
\mathcal{L}_{\CDRFT}(\theta)=\Lpost(\theta)+\lambda\,\Bshared\!\big(\Cmat(\theta)\big),\qquad\lambda>0,
\label{eq:objective}
\end{equation}
so that lowering $\Bshared(\Cmat)$ at a fixed activational ceiling enlarges the gap in Eq.~\eqref{eq:acg}. It attaches to any reference-constrained objective; we use GRPO \citep{shao2024deepseekmath}. Computing $\nabla_\theta\Bshared(\Cmat)$ exactly requires differentiating the already first-order control coefficient once more in $\theta$, i.e.\ second-order automatic differentiation: eager-only, memory-exploding, and incompatible with flash-attention \citep{dao2022flashattention} and parameter sharding. The next three subsections reduce its $\theta$-gradient to one backward pass over a stop-gradient control estimate (operation count in Appendix Remark~\ref{app:remC1}). Table~\ref{tab:decomp} checks which of the two components of $\Bshared(\Cmat)$ the regularizer moves: both fall.

\begin{table}[t]
\centering

\small
\setlength{\tabcolsep}{5pt}
\begin{tabular}{lcccc}
\toprule
\addlinespace[1pt]
 & GRPO & \CDRFT{} & $\Delta$ & $\Delta<0$ \\
\addlinespace[1pt]
\midrule
\addlinespace[3pt]
$\Bshared(\Cmat)$ & 82.8 & 67.3 & $-15.5\pm6.5$ & 6/7 \\
$\Bdir(\Cmat)$ & 58.2 & 50.6 & $-7.6\pm2.1$ & 7/7 \\
$\Bnorm(\Cmat)$ & 77.3 & 62.9 & $-14.4\pm7.1$ & 6/7 \\
\bottomrule
\end{tabular}
\caption{Decomposition of the $\Bshared(\Cmat)$ drop on a matched GRPO/\CDRFT{} pair ($\beta{=}0$; seven probe seeds, same-seed paired differences). $\Bdir(\Cmat)=\lambda_{\max}(R)/M$, with $R$ the matrix of row cosines, is $\Bshared$ on a row-normalized $\Cmat$ and so measures direction sharing alone; $\Bnorm(\Cmat)$ is the value $\Bshared$ takes at exactly orthogonal rows. Only the direction term falls on every seed (Appendix~\ref{sec:app_rownorm}).}
\label{tab:decomp}
\end{table}

\subsection{A Stable Spectral-Moment Ratio Replacing the Spectral Extremum}\label{sec:momentratio}

Optimizing the spectral extremum $\lambda_{\max}$ directly invokes an eigendecomposition of $G_C$ that is non-smooth at top-eigenvalue degeneracy $\lambda_1=\lambda_2$, where the eigenvector derivative blows up as $1/(\lambda_1-\lambda_i)$, as Appendix Proposition~\ref{app:propC1} shows. Borrowing from spectral graph convolution the idea of avoiding eigendecomposition by polynomial approximation \citep{defferrard2016convolutional,kipf2017gcn}, and writing the normalized spectrum $\hat\lambda_i=\lambda_i/\sum_j\lambda_j$, we substitute the stable spectral-moment ratio
\begin{equation}
\Rproxy(\Cmat)=\frac{\operatorname{tr}(G_C^2)}{\operatorname{tr}(G_C)^2}=\sum_{i=1}^M\hat\lambda_i^2.
\label{eq:momentratio}
\end{equation}
This quantity is a pure polynomial (matrix multiplication, no eigendecomposition), is $C^\infty$ at $\Cmat\neq 0$ with no $1/(\lambda_1-\lambda_i)$ singularity, and is co-monotone with $\Bshared=\max_i\hat\lambda_i$ under the majorization order, a partial order, sharing the bounds $\{1,1/M\}$, as Appendix Proposition~\ref{app:propC2} establishes. Replacing $\Bshared$ by $\Rproxy$ buys a smooth objective with bounded gradient at degeneracy; both share extrema, and we verify their agreement empirically during training (Figure~\ref{fig:proxydist}b).

\subsection{Closed-Form Sensitivity and the Irreplaceability of the Control Axis}\label{sec:sensitivity}

Let $\bar C$ be the control matrix obtained from one ordinary backward pass with the gradient with respect to $\theta$ stopped, and write $\bar G=\bar C\bar C^\top$. The sensitivity of the spectral-moment ratio to the control matrix admits a closed form, with no eigendecomposition,
\begin{equation}
W:=\frac{\partial\Rproxy}{\partial \Cmat}\bigg|_{\bar C}=\frac{4}{\operatorname{tr}(\bar G)^2}\big(\bar G\,\bar C-\bar{\Rproxy}\,\operatorname{tr}(\bar G)\,\bar C\big)\in\mathbb R^{M\times K}
\end{equation}
Appendix Lemma~\ref{app:lemC1} gives the derivation. By the chain rule, the true gradient is the projection of this stop-gradient weight $W$ along the direction of the control matrix, $\nabla_\theta\Rproxy=\nabla_\theta\langle W,\,\Cmat(\theta)\rangle$; hence $W$ must act on the control matrix $\Cmat(\theta)$ rather than on the activation usage $A_k=\mathbb E\lVert\mathrm{out}_k\rVert_2$.

\begin{theorem}[irreplaceability of the control axis]\label{thm:controlaxis}
Substituting the activation usage $A_k$ for $\Cmat(\theta)$ fails on two structural grounds: it violates the premise $F\neq \Cmat$ of Lemma~\ref{lemma:axes}; and the per-gate scalar sums out the family dimension, so it cannot express per-task decoupling.
\end{theorem}
Appendix Theorem~\ref{app:thmC1} proves this.
Empirically, minimizing $\sum_k w_k A_k$ under fixed reward further cuts the high-usage shared directions and forces control to reconcentrate, the concentration spike observed on large models by early implementations that act on the activation usage.

\subsection{First-Order Proxy Gradient via Central Differences in Gate Space}\label{sec:centraldiff}

Expanding the Frobenius inner product row by row into gate-space directional derivatives gives
\begin{equation}\label{eq:rowexpand}
\langle W,\Cmat(\theta)\rangle=\sum_{m}\frac1{\ell_m}D_{W_{m,\cdot}}\ell_m,
\end{equation}
where $D_{W_{m,\cdot}}\ell_m=\langle W_{m,\cdot},\,\partial_g \ell_m|_{g=\mathbf 1}\rangle$ is the directional derivative of $\ell_m$ along $W_{m,\cdot}$. Since the gate-space dimension $K$ is far smaller than $\dim\theta$, we estimate this directional derivative by a central difference in gate space.

\begin{theorem}[central-difference form of the proxy gradient]\label{thm:centraldiff}
Writing $\hat W_{m,\cdot}=W_{m,\cdot}/\lVert W_{m,\cdot}\rVert_2$, we have
\begin{equation}
\label{eq:centraldiff}
\resizebox{\columnwidth}{!}{$\displaystyle
\langle W,\Cmat(\theta)\rangle=\sum_{m=1}^M\frac{\lVert W_{m,\cdot}\rVert_2}{\ell_m}\cdot\frac{\ell_m(\theta;\mathbf 1+\epsilon\hat W_{m,\cdot})-\ell_m(\theta;\mathbf 1-\epsilon\hat W_{m,\cdot})}{2\epsilon}+O(\epsilon^2),
$}
\end{equation}
where the two gated forward passes carry only constant gate scalings and are ordinarily differentiable in $\theta$; hence $\nabla_\theta\Lproxy$ is obtained by one backward pass, with no second-order graph, and agrees with $\nabla_\theta\Rproxy$ to within $O(\epsilon^2)$.
\end{theorem}
Appendix Theorem~\ref{app:thmC2} proves this.

Intuitively, the gradient needs only the projection of $\Cmat(\theta)$ along the fixed $W$, and a directional derivative along a fixed gate direction is the slope of a scalar read out by two forward passes; differentiation with respect to the gates is absorbed into the forward passes, leaving a first-order graph in $\theta$. Numerically we perturb along the unit direction and rescale by $\lVert W_{m,\cdot}\rVert_2$ to stay in the linear regime, as Appendix Remark~\ref{app:propC4} notes.

\subsection{Estimation Protocol}\label{sec:estimation}

The control matrix is estimated on a fixed probe of $n$ samples per family, each row a family mean, so $\Bshared(\Cmat)$ is an $n$-dependent estimator: under the rank-one signal-plus-noise model of Appendix Proposition~\ref{app:propC5},
\begin{equation}\label{eq:estimator}
\mathbb E\big[\operatorname{tr}(G_C)\big]=\underbrace{s^2 M_{\mathrm{fam}}}_{\text{signal}}+\underbrace{\tfrac1n\operatorname{tr}(\Sigma_N)}_{\text{per-row noise}},
\end{equation}
so a small probe inflates the denominator and under-reads concentration, while the leading direction is located at a sample size independent of $K$ (Davis--Kahan \citep{davis1970rotation}). We therefore fix $n$ across compared models and cut variance by offline seed averaging rather than a larger probe.

\begin{figure}[t]
\centering
\includegraphics[width=\linewidth]{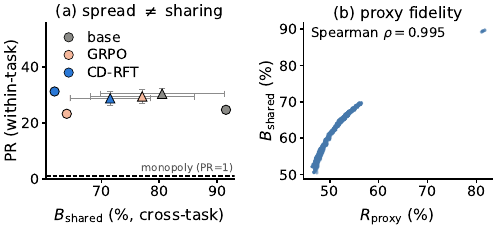}
\caption{Distributedness of control and proxy fidelity. \textbf{(a)} Every family spreads control over many gates ($\mathrm{PR}\gg1$), equally at every $\Bshared(C)$, so no per-gate scalar characterizes control. Circles Qwen2.5-7B, triangles Llama-3.2-3B (5 seeds). \textbf{(b)} $\Rproxy$ tracks $\Bshared(\Cmat)$ in training (\CDRFT{}, $\beta{=}0$).}
\label{fig:proxydist}
\end{figure}

\subsection{Implementation and Scalability}\label{sec:impl}

The soft coefficient $\lambda$ alone cannot curb the transient concentration peak during the reinforcement-learning warm-up, so we enforce the near-hard constraint
\begin{equation}\label{eq:constraint}
\min_\theta\ \Lpost(\theta)\quad\text{s.t.}\quad \Bshared\!\big(\Cmat(\theta)\big)\le\tau,
\end{equation}
by a per-step inner loop that descends $-\nabla_\theta\Lproxy$ until $\Bshared\le\tau$ or a cap $K_{\max}$ is reached; Appendix Algorithm~\ref{alg:cdrft} and Definition~\ref{app:defC3} give the full step and the $(\tau,K_{\max})$ setting. Reduced to a single backward pass by Theorem~\ref{thm:centraldiff}, the regularizer is compatible with flash-attention and parameter sharding, scales to full fine-tuning of 7B on a single GPU without low-rank adaptation, and adds only a minor per-step overhead, as Appendix Remark~\ref{app:propC7} shows.

\section{Experiments}\label{sec:exp}

We organize the experiments around two questions: (i) whether the Shared Control Bottleneck $\Bshared(\Cmat)$ is the right object for characterizing the internal changes wrought by reinforcement-learning post-training, namely whether it yields a reproducible mechanistic signature that activation-level metrics cannot; and (ii) whether explicitly decoupling this bottleneck on top of the same recipe (\CDRFT{}) predictably translates into gains in multi-task capability.

\subsection{Experimental Setup}\label{sec:setup}

We train Qwen2.5-7B \citep{yang2024qwen25} from the base model with GRPO \citep{shao2024deepseekmath} as the backbone, under an RLVR setting \citep{lambert2024tulu3}, on a balanced mixture of three families of verifiable tasks (mathematics, code, and logic) drawn respectively from DeepScaleR \citep{luo2025deepscaler}, code-r1-12k \citep{liu2025coder1}, and the GURU logic subset \citep{cheng2025guru}, each scored by a rule or test-case verifier, with data and mixing ratios in Appendix~\ref{sec:appendix}. To isolate the single variable of adding the control-diverse regularizer, we compare along two paired axes, without KL ($\beta{=}0$) and with KL ($\beta{=}0.01$), on each of which \CDRFT{} and its matched GRPO share every setting except the regularizer; the untrained base serves as an anchor.

Evaluation uses the nine held-out benchmarks of Table~\ref{tab:main}, reported with unbiased $\passk$ \citep{chen2021evaluating} at a fixed temperature applied identically to all methods. The full setup is deferred to the appendix.

\subsection{Mechanistic Diagnosis of the Shared Control Bottleneck}\label{sec:mechanism}

\subsubsection{Activation-Level Metrics Are Insufficient.}
Two of the three activation metrics of \citet{zhang2025reinforcement} reproduce robustly (reinforcement-learning fine-tuning raises activation intensity and lowers distribution kurtosis across all methods and datasets), but the third, information complexity (the ``activation diversity'' metric), is \textbf{direction-unstable}: as Figure~\ref{fig:info} shows, among models of comparable capability it stays high under \CDRFT{} yet nearly collapses under the matched GRPO, in the same direction across all three domains and by several standard errors. Without overturning the intensity and kurtosis effects that do replicate, the component singled out as evidence of \emph{activation diversity} disagrees with itself on equally healthy models, so activation-level metrics cannot on their own answer what reinforcement learning changes. The object of analysis should move from which components are activated to which control the reward flux, the control matrix and its cross-task concentration $\Bshared(\Cmat)$.

\begin{figure}[t]
\centering
\includegraphics[width=\linewidth]{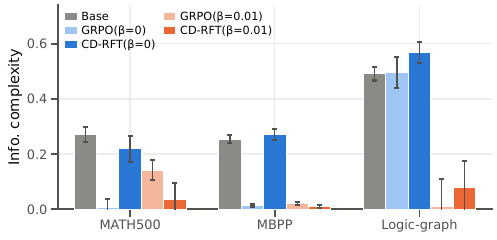}
\caption{Information complexity (edge-distribution entropy) across the three domains, with bootstrap confidence intervals.}
\label{fig:info}
\end{figure}

\subsubsection{The Shared Control Bottleneck Yields a Robust Signature.}
We measure $\Bshared(\Cmat)$ under the probe protocol detailed in the appendix, reported in Table~\ref{tab:main}. The activation axis $\Bshared(\Fmat)$ stays at the reference ceiling of Section~\ref{sec:acg} across all methods, so we analyze the control axis only.

Table~\ref{tab:main} shows a clean ordering, base $>$ the RL baselines $>$ the \CDRFT{} variants: reinforcement learning already lowers control sharing, and the regularizer lowers it further and most, so both variants open a larger Activation--Control Gap than their matched baseline. Because absolute values depend on probe sampling, the method claim is read as the same-seed paired difference of Table~\ref{tab:decomp} and Appendix Figure~\ref{fig:mechanism}: robust on the no-KL axis and directionally consistent, if smaller, with KL, matching the coverage-side gains of the latter.

Across five base models spanning two families and several scales, Appendix Figure~\ref{fig:crossmodel} shows $\Bshared(\Cmat)$ above the $99$th percentile of a simulated random reference at this shape ($42.8\pm3.4$, well above the bound $1/M$; Appendix~\ref{sec:app_mc}), so the bottleneck is not a single-model artifact. The two models we intervene on (Qwen2.5-7B and Llama-3.2-3B; Section~\ref{sec:robustness}) have the deepest bottlenecks, and hence the most decoupling headroom; absolute values are tokenizer- and probe-dependent, so the comparison is read across models by ordering.

\subsection{Multi-Task Capability}\label{sec:capability}

We evaluate capability on nine benchmarks and summarize by domain in Table~\ref{tab:main}. The paired comparison is a controlled test of the diagnosis: if the Shared Control Bottleneck is what matters, decoupling it on an otherwise identical recipe should move capability. We read the net change of \CDRFT{} relative to its matched GRPO on each $\beta$ axis, an ordering the appendix temperature and checkpoint sweeps confirm.

\begin{table*}[t]
\centering

\small
\setlength{\tabcolsep}{7pt}
\begin{tabular}{lcccccccccccc}
\toprule
\addlinespace[1pt]
& \multicolumn{2}{c}{mechanism} & \multicolumn{5}{c}{\passone} & \multicolumn{5}{c}{\passk} \\
\cmidrule(lr){2-3} \cmidrule(lr){4-8} \cmidrule(lr){9-13}
Method & $\Bshared\downarrow$ & $\ACG\uparrow$ & math & code & logic & \textbf{overall} & hard & math & code & logic & \textbf{overall} & hard \\
\addlinespace[1pt]
\midrule
\addlinespace[3pt]
Base & 83.4\,$\pm$\,3.8 & 16.2 & 20.0 & 54.5 & 33.2 & 35.9 & 6.8 & 71.3 & 92.0 & 90.8 & 84.7 & 56.6 \\
\addlinespace[0.25em]
GRPO ($\beta{=}0$) & 81.1\,$\pm$\,4.2 & 18.6 & 23.5 & 58.0 & 39.7 & 40.4 & 8.1 & 72.1 & 92.9 & 94.1 & 86.4 & 57.0 \\
\ourrow \addon{\CDRFT} & \textbf{68.9\,$\pm$\,3.8} & \textbf{30.7} & \textbf{24.0} & \textbf{60.5} & \textbf{40.5} & \textbf{41.7} & \textbf{8.3} & 74.2 & \textbf{93.4} & \textbf{94.8} & 87.5 & 60.6 \\
\addlinespace[0.25em]
GRPO ($\beta{=}0.01$) & 79.5\,$\pm$\,5.0 & 20.1 & 21.5 & 59.1 & 36.5 & 39.0 & 7.0 & 70.5 & 92.7 & 92.1 & 85.1 & 55.3 \\
\ourrow \addon{\CDRFT} & 73.4\,$\pm$\,5.0 & 26.3 & 22.3 & 59.9 & 36.4 & 39.5 & 7.4 & \textbf{76.0} & 93.1 & 93.5 & \textbf{87.6} & \textbf{63.4} \\
\bottomrule
\end{tabular}
\caption{Mechanism and capability on Qwen2.5-7B. $\Bshared$ abbreviates $\Bshared(\Cmat)$; it and $\ACG$ are measured on the fixed probe over seven probe seeds (mean$\pm$SE), while the capability columns are benchmark accuracies. Math = MATH500/AMC23/AIME24/AIME25/Minerva; code = HumanEval+/MBPP; logic = ordering-puzzle/Logic-graph; hard = AIME24/AIME25/Minerva. Each domain averages its sets; overall averages the three domains; all entries are percentages; full ladder in Appendix Table~\ref{tab:ladder}.}
\label{tab:main}

\end{table*}

Two trends stand out. First, adding the control-diverse regularizer improves both accuracy axes over its matched GRPO baseline on almost every benchmark. As Appendix Figure~\ref{fig:gains} shows, the no-KL variant improves greedy $\passone$ on 8 of 9 benchmarks, and the with-KL variant improves $\passk$ coverage on 8 of 9, so the two variants are complementary (one leads on $\passone$, the other on coverage), and both dominate the base across all three domains. The gains are those of a well-tuned strong-baseline comparison, not of over-tuning: the $\passone$ margin over the matched RL baseline stays within a few points, while the coverage advantage of the with-KL variant grows steadily with $k$ and is largest on the hardest problems, as Appendix Figure~\ref{fig:passk} shows.

Second, the coverage shortfall at large $k$ is a known failure of RLVR: reward maximization trades large-$k$ coverage for $\passone$, and the boundary can fall to or below the base \citep{yue2025rl}. It tracks the reward objective, both matched GRPO baselines converging to the base on the hardest sets (Appendix Table~\ref{tab:ladder}) at $\beta{=}0$ and $\beta{=}0.01$ alike, so it lies in the reward flux of Eq.~\eqref{eq:flux} that the analysis targets. \CDRFT{} addresses it in control space, where the diagnosis locates the bottleneck: on both axes it restores coverage to the base level or above by holding control decoupled while the reward flux is optimized, the behavioural counterpart of the mechanism of Section~\ref{sec:mechanism} and, by construction, visible only at large $k$.

\subsection{Ablation: Net Effect of the Regularizer}\label{sec:ablation}

To isolate the effect of the control regularizer from backbone differences, we toggle only the on/off of the regularizer under a fully balanced setting. This on/off contrast is precisely the matched GRPO\,$\to$\,\CDRFT{} pair of Table~\ref{tab:main}, whose control and capability axes we now read jointly.

Both balanced pairs trace the same pattern: adding the regularizer lowers $\Bshared(\Cmat)$ (mechanistic decoupling), and overall $\passone$ and $\passk$ rise in tandem. The chain is tightest where the compression is strongest (the no-KL axis, whose larger drop in $\Bshared(\Cmat)$ coincides with the largest $\passone$ gain), while the with-KL axis compresses less but gains more on coverage. The regularizer thus acts on the control axis, and the capability gain moves with the decoupling rather than with the backbone or hyperparameters.

\subsection{Robustness and Generalization}\label{sec:robustness}

We test that the main result holds along three axes over which one might worry it was tuned: the evaluation temperature, the method's own target hyperparameter, and the model family.

\paragraph{Sampling temperature.}
We repeat the evaluation over a per-domain temperature sweep bracketing the main-table working temperature of each domain. On the math hard set and on code, the ordering ``\CDRFT{} $\geq$ matched RL'' holds at every temperature on $\passone$ and at nearly all on $\passk$; lowering the temperature raises $\passone$ and raising it raises $\passk$, but the method ordering never flips, so the main-table conclusions are not an artifact of a chosen temperature. The full win-rate table, per-temperature curves, and the per-domain breakdown are in Appendix~\ref{sec:app_temp}.

\paragraph{Proxy target.}
Sweeping the target concentration $\tau\in\{55,65,70\}$ (only $\tau$ changed) traces a smooth $\passone$--$\passk$ trade-off rather than a sharp optimum: the two looser settings are adjacent optima that exchange a little $\passone$ for coverage, and only over-compression ($\tau{=}55$) is clearly worse on both axes. The method is thus insensitive to $\tau$ by design rather than tuning, with the full ladder and per-$\tau$ reading in Appendix Table~\ref{tab:target}.

\paragraph{Second model family.}

Both the mechanism and the capability gain also hold on Llama-3.2-3B \citep{dubey2024llama3}: the control-sharing signature reproduces, and \CDRFT{} leads $\passk$ coverage on every benchmark and $\passone$ on four of five, so the effect is not specific to Qwen2.5-7B. Full protocol, ladders, and the mechanism table are in Appendix~\ref{sec:app_llama}.

\begin{table}[t]
\centering

\small
\setlength{\tabcolsep}{4pt}
\begin{tabular}{lccccc}
\toprule
\addlinespace[1pt]
Method & GSM8K & MATH500 & code & L-ord & L-graph \\
\addlinespace[1pt]
\midrule
\addlinespace[3pt]
\multicolumn{6}{l}{\emph{$\passone$}} \\
Base & 10.0 & 5.9 & 4.0 & 2.6 & 1.5 \\
GRPO & 16.4 & 7.4 & 27.7 & \textbf{11.7} & 29.5 \\
\ourrow \addon{\CDRFT} & \textbf{16.8} & \textbf{7.6} & \textbf{28.2} & 11.0 & \textbf{29.6} \\
\addlinespace[0.35em]
\multicolumn{6}{l}{\emph{$\passk$ }} \\
Base & 79.6 & 47.8 & 51.4 & 31.0 & 46.0 \\
GRPO & 83.0 & 48.4 & 69.0 & 37.0 & 70.0 \\
\ourrow \addon{\CDRFT} & \textbf{84.2} & \textbf{51.2} & \textbf{69.9} & \textbf{46.0} & \textbf{71.0} \\
\bottomrule
\end{tabular}
\caption{Transfer to Llama-3.2-3B (base, $\beta{=}0$). Each shaded +\,\CDRFT{} row is the add-on over the GRPO above it; L-ord = Logic-ordering, L-graph = Logic-graph; all columns are single benchmarks except code, the mean of HumanEval+ and MBPP. All entries are percentages.}
\label{tab:llama}

\end{table}

\section{Conclusion }\label{sec:conclusion}

We have reframed the interpretability of reinforcement-learning post-training: the question is not which components are activated, but which control the reward gain, and whether that control stays task-specific or collapses into one shared channel. This shift from an activation view to a control view turns a diagnosis into a lever, writing the Shared Control Bottleneck into the loss, \CDRFT{} decouples control at low overhead and improves multi-task \passone{} and \passk{} across two model families. We see the control view as the broader contribution: once post-training is read as reshaping which components control the reward flux, control-level objectives become a natural handle on how a finite capability is split among competing tasks, beyond the multi-task rule-verifier setting studied here.

\bigskip
\bibliography{references}

\clearpage
\onecolumn             
\appendix
\numberwithin{proposition}{section}
\numberwithin{theorem}{section}
\numberwithin{lemma}{section}
\numberwithin{corollary}{section}
\numberwithin{assumption}{section}
\numberwithin{definition}{section}
\numberwithin{remark}{section}
\numberwithin{figure}{section}
\numberwithin{table}{section}
\CDenableToc            
\setcounter{page}{1}   
\setcounter{tocdepth}{2}
\begin{center}\Large\textbf{APPENDIX}\end{center}
\tableofcontents
\vspace{1em}
This appendix has two parts. Appendices~\ref{sec:appendix}--\ref{sec:app_alg} complete the theory: the function-space conventions, the full set of assumptions, the secondary propositions, and the proofs abbreviated in the main text, with Appendix~\ref{sec:appendix} corresponding to the setting (Section~\ref{sec:prelim}), Appendix~\ref{sec:app_diag} to the problem (Section~\ref{sec:motivation}), Appendix~\ref{sec:app_method} to the method (Section~\ref{sec:method}), and Appendix~\ref{sec:app_alg} giving one training step in full. Appendices~\ref{sec:settings}--\ref{sec:overhead} complete the experiments: Appendix~\ref{sec:settings} fixes the settings shared by every reported number, Appendix~\ref{sec:app_mech} the mechanistic results deferred from Section~\ref{sec:mechanism}, Appendix~\ref{sec:app_cap} the per-benchmark capability results deferred from Section~\ref{sec:capability}, Appendix~\ref{sec:app_temp} the temperature sweep and Appendix~\ref{sec:app_target} the proxy-target sweep of Section~\ref{sec:robustness}, Appendix~\ref{sec:app_ckpt} the cross-checkpoint robustness check, Appendix~\ref{sec:app_llama} the second-model (Llama-3.2-3B) transfer results of Section~\ref{sec:robustness}, and Appendix~\ref{sec:overhead} the measured cost of the regularizer.

\section{Full Specification of the Control Coefficient}\label{sec:appendix}

\paragraph{Notation and function space.}
The vocabulary $\mathcal V$ is finite, with $\mathcal V^\star=\bigcup_{T\ge0}\mathcal V^T$; a prompt is $x\in\mathcal X$ and a response is $y\in\mathcal V^\star$. The policy factorizes autoregressively as $\pi_\theta(y\mid x)=\prod_{t}\pi_\theta(y_t\mid x,y_{<t})$, with parameters $\theta\in\Theta\subseteq\real^P$ (an open set) and reference policy $\pi_0=\pi_{\theta_0}$. The residual-stream dimension is $d$; the state at layer $\ell$ is $h^\ell\in\real^d$; circuit unit $k$ resides at layer $\ell(k)$ with map $f_k(\cdot;\theta):\real^d\to\real^d$. The gradient of the target log-likelihood with respect to the final residual is the \textbf{readout direction} $w_m:=\nabla_{h^L}\sum_t\log\pi_\theta(y^\star_t\mid x,y^\star_{<t})\in\real^d$ (the linear readout direction, at the final layer, of the metric used by attribution patching); the update Jacobian of layer $\ell'$ is $J_{\ell'}:=\partial\big(\sum_{k:\ell(k)=\ell'}f_k\big)/\partial h^{\ell'}$. The control matrix has $M=M_{\mathrm{fam}}$ rows (one per task family, each averaged over its $n$ probe samples) and $K$ columns for the gated components.

\begin{assumption}[smoothness and integrability]\label{app:asmA0}
Write $\pi_\theta(y\mid x;g)$ for the gated policy of Eq.~\eqref{eq:gating}. For almost all $(x,y^\star)$ the map $(\theta,g)\mapsto\log\pi_\theta(y^\star\mid x;g)$ is smooth ($C^\infty$); for each $m$, $\big|\log\pi_\theta(y^\star\mid x;g)\big|$ and its partial derivatives in $\theta$ and in $g$ are dominated by $\Dm$-integrable functions, so that expectation and differentiation commute. Hence $\ell_m(\theta;g)=\mathbb E_{\Dm}[\log\pi_\theta(y^\star\mid x;g)]$ is smooth in both arguments, and so is $\Jm(\theta;g)$, which differs from $\ell_m(\theta;g)$ by the gate-independent constant $\mathbb E_{\Dm}[\log\pi_0(y^\star\mid x)]$: every statement below may therefore be read for either object.
\end{assumption}

\begin{assumption}[nominal identity]\label{app:asmA1}
$\pi_\theta(\cdot\mid\cdot\,;\1)=\pi_\theta$: at the nominal gate value the gated network is the original model. The gating is thus an analysis probe, and all control quantities are differentiated in a neighborhood of $g=\1$.
\end{assumption}

\begin{remark}[scope of the smoothness assumption]\label{app:remA0}
Assumption~\ref{app:asmA0} holds for transformers assembled from smooth components---SiLU/SwiGLU, softmax, RMSNorm---and the gates of Eq.~\eqref{eq:gating} enter multiplicatively, so they preserve smoothness; every model studied here (Qwen2.5, Llama-3.2) is of this kind. Architectures with ReLU-type activations lie outside the stated scope. Every differentiability order used later follows: the second-order Taylor expansion of Proposition~\ref{app:propA1}, the third-order central-difference remainder of Theorem~\ref{app:thmC2}, and the real-analyticity invoked in Theorem~\ref{app:thmB1}. The integrability conditions of Assumption~\ref{app:asmA0} are a separate requirement and are not implied by smoothness.
\end{remark}

\begin{assumption}[normalizer non-degeneracy]\label{app:asmA2}
There exists $\epsilon_0>0$ (set to $10^{-4}$ in the implementation) such that every family admitted for analysis satisfies $|\ell_m(\theta)|\ge\epsilon_0$; otherwise Eq.~\eqref{eq:coeff} is unidentifiable and the family is discarded.
\end{assumption}

\begin{remark}[why normalize by $\ell_m$, not $J_m$]\label{app:remA1}
MCA customarily writes the coefficient as $\partial\log J_m/\partial\log g_k$, presuming a positive flux. Here $J_m$ is a sign-indefinite log-likelihood-ratio margin (undefined under $\log$ when $J_m\le0$) and can approach zero, so we normalize the shared gate-derivative $\partial_{g_k}J_m=\partial_{g_k}\ell_m$ by the target log-likelihood $\ell_m<0$, which stays bounded away from zero; the scaled form of Eq.~\eqref{eq:coeff} reduces to the classical log--log derivative only in the positive-flux regime. We also fix the reach of $\Jm$: it is a teacher-forced likelihood margin on one reference target, not the sampling-time expected verifier reward, and raising it on a single $y^\star$ need not raise coverage. The control coefficient differentiates this margin, and the capability claims of Section~\ref{sec:capability} rest on the benchmark $\passk$ measurements, not on $\Jm$.
\end{remark}

\begin{remark}[the guard rarely binds; row signs are immaterial]\label{app:remA2}
Assumption~\ref{app:asmA2} does not bind in practice: no family is discarded in our offline or training estimations, the smallest observed $|\ell_m|$ exceeding $\epsilon_0$ by more than an order of magnitude. The sign of a control row is likewise immaterial to the bottleneck, since $\Bshared$ depends on $\Cmat$ only through $\Cmat\Cmat^\top$, which is invariant under a sign flip of any row.
\end{remark}

\begin{remark}[attribution patching]\label{app:remA3}
For a metric $\mathcal M$, the single-backward readout of the control coefficient linearizes an activation replacement $a_k\!\leftarrow\!a_k'$ by
\begin{equation}\label{app:eq:ap}
\mathcal M(a_k\!\leftarrow\!a_k')-\mathcal M(a_k)\approx\big\langle\nabla_{a_k}\mathcal M,\ a_k'-a_k\big\rangle,
\end{equation}
here with $\mathcal M=\ell_m$.
\end{remark}

\begin{proposition}[first-order expansion; main text Eq.~\eqref{eq:expansion}]\label{app:propA1}
Under Assumptions~\ref{app:asmA0}--\ref{app:asmA2}, with $u=\log g$,
\[
\frac{\ell_m(\theta;e^u)-\ell_m(\theta)}{\ell_m(\theta)}=\sum_k C_{m,k}\,u_k+O(\lVert u\rVert_2^2).
\]
\end{proposition}
\begin{proof}
Let $\widetilde \ell_m(u)=\ell_m(\theta;e^u)$, which is smooth by Assumption~\ref{app:asmA0}. The chain rule gives $\partial_{u_k}\widetilde \ell_m|_0=\partial_{g_k}\ell_m|_{g=\1}\cdot e^{u_k}|_0=\partial_{g_k}\ell_m|_{g=\1}$. A first-order Taylor expansion with Peano remainder, divided by $\ell_m(\theta)\neq 0$ (Assumption~\ref{app:asmA2}), and substituting $C_{m,k}=\tfrac1{\ell_m}\partial_{g_k}\ell_m|_{g=\1}$, yields the claim.
\end{proof}

\begin{corollary}[logarithmic form]\label{app:corA1}
Since $\ell_m<0$, writing $z=\delta\ell_m/\ell_m$ and $\log(1+z)=z+O(z^2)$ gives $\delta\log|\ell_m|=\sum_k C_{m,k}\,\delta\log g_k+O(\lVert\delta\log g\rVert_2^2)$, the sign-robust analogue of the classical MCA log--log derivative.
\end{corollary}

\begin{lemma}[activation axis and control axis are not interchangeable]\label{app:lemA1}
There exist networks and task families for which no scalar $c_m>0$ satisfies $C_{m,\cdot}=c_m F_{m,\cdot}$.
\end{lemma}
\begin{proof}
$J_m$ is determined by the projection of $h^L$ onto $w_m$. The first-order effect of gate $g_k$ on $J_m$ propagates through Eq.~\eqref{eq:gating} to $h^L$ and is read along $w_m$. Choose a circuit $k$ whose write $\mathrm{out}_k\perp w_m$ and whose downstream Jacobian does not rotate it into $w_m$; then $\partial_{g_k}J_m|_{g=\1}=0$, hence $C_{m,k}\approx 0$, while $F_{m,k}=\mathbb E\lVert\mathrm{out}_k\rVert_2$ can be made independently large. Dually, a circuit with small write aligned with $w_m$ has low usage and high control. The coexistence makes the large-component supports of $C_{m,\cdot}$ and $F_{m,\cdot}$ differ, excluding any positive proportionality.
\end{proof}

\section{Full Specification of the Shared Control Bottleneck}\label{sec:app_diag}

\begin{proposition}[failure of the summation theorem]\label{app:propB1}
In general, $\sum_k C_{m,k}\neq 1$.
\end{proposition}
\begin{proof}
By Eq.~\eqref{eq:coeff}, $\sum_k g_k\partial_{g_k}\ell_m|_{g=\1}=\ell_m\sum_k C_{m,k}$. If $g\mapsto \ell_m(\theta;g)$ were first-order homogeneous at $g=\1$, Euler's theorem would give the left-hand side $=\ell_m$, whence $\sum_k C_{m,k}=1$. But the residual update in Eq.~\eqref{eq:gating} contains the identity term $h^\ell$: under uniform scaling $g\mapsto\alpha g$, $h^\ell$ does not scale with $\alpha$, and layer normalization is not homogeneous, so $\ell_m(\theta;\alpha g)$ is not first-order homogeneous in $\alpha$, the left-hand side $\neq \ell_m$, and therefore $\sum_k C_{m,k}\neq 1$.
\end{proof}

\begin{remark}[abandoning the coverage constraint]\label{app:corB1}
Since $\sum_k C_{m,k}$ is not structurally centered at $1$, $(\sum_k C_{m,k}-1)^2$ has no basis as a coverage constraint and is used only as a diagnostic.
\end{remark}

\begin{theorem}[no single-gate monopoly; main text Theorem~\ref{thm:nosingle}]\label{app:thmB1}
Under generic weights, $\partial_{g_k}J_m|_{g=\1}\neq0$ for every $k$.
\end{theorem}
\begin{proof}
Differentiating for $k\notin S$, gate $g_k$ writes $f_k(h^{\ell(k)};\theta)$ at its resident layer and propagates through the subsequent layer Jacobians to $h^L$:
\[
\frac{\partial J_m}{\partial g_k}\bigg|_{g=\1}=w_m^\top\Big(\textstyle\prod_{\ell'=\ell(k)+1}^{L-1}(I+J_{\ell'})\Big)f_k(h^{\ell(k)};\theta),
\]
with $J_{\ell'}$ the layer update Jacobian defined above. Vanishing at a given $k$ requires the propagated write to lie in $w_m^\perp$; for each $k$, with real-analytic activations this is the zero set of a not-identically-zero real-analytic function of $\theta$, hence a measure-zero set in parameter space, and a finite intersection remains measure zero. Hence no gate derivative vanishes under generic weights, and in particular the control support is not a single gate.
\end{proof}

\begin{proposition}[bounds, saturation, invariance of $\Bshared$]\label{app:propB2}
Let $G_X$ have eigenvalues $\lambda_1\ge\dots\ge\lambda_M\ge0$. Then (i) $1/M\le\Bshared\le1$; (ii) $=1\iff\operatorname{rank}(X)=1$ (collinear rows); (iii) $=1/M\iff\lambda_1=\dots=\lambda_M$ (orthogonal, equal-norm rows); (iv) invariant under $X\mapsto cX$ ($c\neq0$) and $X\mapsto XQ$ ($QQ^\top=I$).
\end{proposition}
\begin{proof}
(i)--(iii): $\Bshared=\lambda_1/\sum_i\lambda_i$; $\sum_i\lambda_i\le M\lambda_1$ gives the lower bound (equality $\iff$ all eigenvalues equal), and $\lambda_1\le\sum_i\lambda_i$ the upper bound (equality $\iff\lambda_{\ge2}=0\iff$ rank one $\iff$ collinear rows). (iv): $G_{cX}=c^2G_X$ scales all eigenvalues by $c^2$, leaving the ratio unchanged, while $G_{XQ}=XQQ^\top X^\top=G_X$.
\end{proof}

\begin{proposition}[gap range and characterization; main text Proposition~\ref{prop:pathology}]\label{app:propB3}
As $\Bshared(F)\to1$: (i) $\ACG\in[0,\Bshared(F)-1/M]\to[0,1-1/M]$; (ii) $\ACG\to0\iff\Bshared(C)\to\Bshared(F)$; (iii) $\ACG$ is large $\iff\Bshared(C)$ lies far below the ceiling.
\end{proposition}
\begin{proof}
The range follows from Proposition~\ref{app:propB2}(i), $\Bshared(C)\in[1/M,1]$, with Eq.~\eqref{eq:acg}; parts (ii)--(iii) are read from Eq.~\eqref{eq:acg} and given mechanistic meaning via Proposition~\ref{app:propB2}(ii)--(iii).
\end{proof}

\begin{remark}[correction to the narrative]\label{app:remB1}
Empirically, the $\Bshared(C)$ of healthy reinforcement learning is already below that of the base model; reinforcement learning itself reduces control sharing. Our claim is therefore controlled-variable: on top of a matched recipe, our method depresses $\Bshared(C)$ the most.
\end{remark}

\begin{proposition}[perturbation variance]\label{app:propB4}
When the gate logarithms are perturbed by $\xi\sim\mathcal N(0,\Sigma)$, the variance of the relative change is
\[
\operatorname{Var}\!\left[\frac{\delta \ell_m}{\ell_m}\right]\approx C_{m,\cdot}^\top\Sigma\,C_{m,\cdot}\ \xrightarrow{\ \Sigma=\sigma^2I\ }\ \sigma^2\lVert C_{m,\cdot}\rVert_2^2.
\]
\end{proposition}
\begin{proof}
By Eq.~\eqref{eq:expansion}, $\delta \ell_m/\ell_m\approx C_{m,\cdot}^\top\xi$; the variance of a zero-mean Gaussian linear form is $C_{m,\cdot}^\top\Sigma C_{m,\cdot}$, which in the isotropic case equals $\sigma^2\lVert C_{m,\cdot}\rVert_2^2$.
\end{proof}

\begin{corollary}[concentration implies fragility]\label{app:corB2}
At fixed total control energy $\operatorname{tr}(G_C)$, the larger $\Bshared(C)$, the more energy concentrates along the leading direction and the larger the perturbation variance of Proposition~\ref{app:propB4} for aligned families.
\end{corollary}
\begin{proof}
Immediate from Proposition~\ref{app:propB4}: at fixed $\operatorname{tr}(G_C)$, raising $\Bshared(C)=\lambda_{\max}/\operatorname{tr}(G_C)$ raises $\lambda_{\max}$, which is the variance coefficient along the leading direction.
\end{proof}
This is a statement about the geometry of $C$ alone; its connection to real deployment (quantization, pruning, out-of-distribution behavior) is left to future work.

\section{Full Specification of the Control-Diverse Regularizer}\label{sec:app_method}

\begin{proposition}[non-smoothness of the spectral extremum at degeneracy]\label{app:propC1}
The spectral extremum has eigenvalue gradient and eigenvector sensitivity
\begin{align*}
\frac{\partial\lambda_1}{\partial C}&=2v_1v_1^\top C,\\
\frac{\partial v_1}{\partial C}&\propto\sum_{i\ge2}\frac{v_iv_i^\top}{\lambda_1-\lambda_i}\quad\text{\citep{magnus2019matrix}};
\end{align*}
the eigenvalue gradient $2v_1v_1^\top C$ stays bounded, but the eigenvector sensitivity diverges at top-eigenvalue degeneracy $\lambda_1=\lambda_2$ as $v_1$ becomes non-unique, so backpropagation through the eigendecomposition is unstable there.
\end{proposition}

\begin{proposition}[legitimacy of the spectral-moment-ratio proxy]\label{app:propC2}
The spectral-moment ratio
\[
\Rproxy(C)=\frac{\operatorname{tr}(G_C^2)}{\operatorname{tr}(G_C)^2}=\sum_i\hat\lambda_i^2
\]
is (i) purely polynomial (matvec, no eigendecomposition); (ii) co-monotone with $\Bshared=\max_i\hat\lambda_i$ with respect to the majorization order, sharing the extremal values $\{1,1/M\}$; (iii) $C^\infty$ at $C\neq0$, free of any $1/(\lambda_1-\lambda_i)$ singularity.
\end{proposition}
\begin{proof}
(ii): both $\sum_i\hat\lambda_i^2$ and $\max_i\hat\lambda_i$ are Schur-convex \citep{marshall2011inequalities} and co-monotone with respect to the spectral majorization order, attaining their lower bound at the uniform spectrum and upper bound at the single-point spectrum; since majorization is only a partial order, spectra incomparable under it may rank differently, so we track empirical agreement of the two during training. (i) and (iii) are immediate.
\end{proof}

\begin{lemma}[closed-form sensitivity; main text Section~\ref{sec:sensitivity}]\label{app:lemC1}
$W=\partial\Rproxy/\partial C|_{\bar C}=\tfrac{4}{\operatorname{tr}(\bar G)^2}\big(\bar G\bar C-\bar{\Rproxy}\operatorname{tr}(\bar G)\bar C\big)$.
\end{lemma}
\begin{proof}
With $\partial\operatorname{tr}(G^2)/\partial C=4GC$ and $\partial\operatorname{tr}(G)/\partial C=2C$, the quotient rule gives
\[
\begin{aligned}
\frac{\partial\Rproxy}{\partial C}
&=\frac{4GC\operatorname{tr}(G)^2-2\operatorname{tr}(G^2)\operatorname{tr}(G)\,2C}{\operatorname{tr}(G)^4}\\
&=\frac{4}{\operatorname{tr}(G)^2}\big(GC-\Rproxy\operatorname{tr}(G)C\big),
\end{aligned}
\]
which, evaluated at $\bar C$, yields the claim.
\end{proof}

\begin{theorem}[irreplaceability of the control axis]\label{app:thmC1}
Replacing $C(\theta)$ by the activation usage $A_k=\mathbb E\lVert\mathrm{out}_k\rVert_2$ fails on two structural grounds: (i) it violates $F\neq C$ (Lemma~\ref{app:lemA1}); (ii) the per-gate scalar $w_k=\sum_m\mathrm{relu}(W_{m,k})$ sums out the family dimension and cannot express per-task decoupling.
\end{theorem}
\begin{proof}
(i) follows from Lemma~\ref{app:lemA1}; (ii) $A_k$ is family-independent and $\sum_m$ discards the per-task directional information in $W$.
\end{proof}
\begin{remark}
Empirically, minimizing $\sum_k w_k A_k$ at fixed reward cuts the usage of the high-flux shared directions used in common across families, forcing re-concentration and raising $\Bshared(C)$; this is the concentration spike seen on large models by early activation-usage implementations, not a claim we establish in closed form.
\end{remark}

\begin{theorem}[single-backpropagation equivalence via central differences; main text Theorem~\ref{thm:centraldiff}]\label{app:thmC2}
Equation~\eqref{eq:centraldiff} holds: the proxy gradient is obtained by a single backpropagation and differs from $\nabla_\theta\Rproxy$ by $O(\epsilon^2)$.
\end{theorem}
\begin{proof}
Decomposing the Frobenius inner product by rows and substituting $C_{m,\cdot}=\tfrac1{\ell_m}\partial_g \ell_m$ gives $\langle W,C\rangle=\sum_m\tfrac1{\ell_m}D_{W_{m,\cdot}}\ell_m$. Expanding $\phi(t)=\ell_m(\theta;\1+t\hat W_{m,\cdot})$ at $0$, we have $\phi(\epsilon)-\phi(-\epsilon)=2\epsilon\phi'(0)+O(\epsilon^3)$ ($\phi$ smooth by Assumption~\ref{app:asmA0}) with $\phi'(0)=D_{\hat W_{m,\cdot}}\ell_m$; multiplying back by $\lVert W_{m,\cdot}\rVert_2$ restores the term with error $O(\epsilon^2)$. The gate scalings $\1\pm\epsilon\hat W$ are constant vectors independent of $\theta$, so the two forward-pass scalars are ordinarily differentiable in $\theta$ and a single backpropagation yields their gradient; the composition equals $\nabla_\theta\langle W,C(\theta)\rangle=\nabla_\theta\Rproxy$, with difference the central-difference remainder $O(\epsilon^2)$.
\end{proof}

\begin{corollary}[directional consistency]\label{app:propC3}
The cosine between the proxy gradient and $\nabla_\theta\Rproxy$ tends to $1$ as $\epsilon\to0$.
\end{corollary}
\begin{proof}
By Theorem~\ref{app:thmC2} the two differ by $O(\epsilon^2)$ while $\nabla_\theta\Rproxy$ is fixed, so the angle between them vanishes with $\epsilon$.
\end{proof}
At toy scale the measured cosine is $\approx1.0$.

\begin{remark}[linear regime and necessity of normalization]\label{app:propC4}
Let $\max_{m,k}|W_{m,k}|=O(\Lambda)$ with $\Lambda\gg1$. One must perturb along the unit direction $\hat W_{m,\cdot}$ and multiply back by $\lVert W_{m,\cdot}\rVert_2$, following $\langle W_{m,\cdot},C_{m,\cdot}\rangle=\lVert W_{m,\cdot}\rVert_2\langle\hat W_{m,\cdot},C_{m,\cdot}\rangle$. Perturbing directly by $\epsilon W_{m,\cdot}$ gives a displacement $\epsilon\lVert W_{m,\cdot}\rVert_2=O(\epsilon\sqrt K\,\Lambda)$ that may greatly exceed $1$, leaving the linear regime and pushing control toward rank-1 collapse.
\end{remark}
Two safeguards follow: skip a family when $\lVert W_{m,\cdot}\rVert_2<10^{-12}$, and when $|\ell_m|<\epsilon_0$ stabilize $1/\ell_m$ by $1/(\ell_m\pm\epsilon_0)$. The working point is $\epsilon=0.05$.

\begin{definition}[inner-loop projection]\label{app:defC3}
Take $\tau\in[1/M,1]$, step size $\eta$, and cap $K_{\max}$. After each main update, iterate on the probe along $-\nabla_\theta\Lproxy$ until $\Bshared(\text{probe})\le\tau$ or $K_{\max}$ is reached. Whereas $\lambda$ governs a soft trade-off, $(\tau,K_{\max},\eta)$ turn it into a near-hard constraint that curbs the transient collinearity spike during warm-up. The working point is $\tau=70$, $K_{\max}=12$, $\eta=4\times10^{-3}$.
\end{definition}

\begin{proposition}[small-probe sufficiency]\label{app:propC5}
Write each family row as a mean over $n$ probe samples, $\Cmat\in\real^{M_{\mathrm{fam}}\times K}$, under a rank-one signal plus per-row noise $C=s\,\1u^\top+N$ ($\lVert u\rVert_2=1$). Then (i) the sample complexity of locating the shared spike $u$ is $\sim1/\mathrm{gap}^2$, independent of $K$ \citep{davis1970rotation}, so a few samples per family stabilize the leading direction; (ii) the concentration ratio $\Bshared(\Cmat)=\lambda_{\max}/\operatorname{tr}(G_C)$ is a biased estimator: per-row noise adds to $\operatorname{tr}(G_C)=\lVert\Cmat\rVert_F^2$ and shrinks as $1/n$, so a small probe raises the denominator and under-reads concentration, making the estimate $n$-dependent.
\end{proposition}
\begin{proof}
(i) is Davis--Kahan applied to the rank-one signal. (ii) With $G_C=s^2M_{\mathrm{fam}}\,\bar{\1}\bar{\1}^\top+(\text{zero-mean cross terms})+NN^\top$, the trace is additive, $\operatorname{tr}(G_C)=s^2M_{\mathrm{fam}}+\operatorname{tr}(NN^\top)$ with $\mathbb E\operatorname{tr}(NN^\top)\propto1/n$ after averaging $n$ samples per row; since $\lambda_{\max}$ is signal-dominated, this extra denominator mass lowers the ratio at small $n$.
\end{proof}

\begin{remark}[consistency of measurement convention]\label{app:corC1}
Family averaging fixes the row dimension at $M=M_{\mathrm{fam}}$, so enlarging $n$ does not change the dimension but only lowers per-row estimation variance, reduced through offline averaging over multiple seeds and small probes; comparisons across models hold $n$ fixed. The working point is $n=3$, an empirical operating point set by the probe compute budget rather than a sufficiency guarantee.
\end{remark}

\begin{remark}[overhead and compatibility]\label{app:propC7}
Once reduced to a single backpropagation via Theorem~\ref{app:thmC2}: (i) the method is compatible with flash-attention and parameter sharding (ZeRO \citep{rajbhandari2020zero} / FSDP \citep{zhao2023fsdp}), and permits full-parameter fine-tuning up to 7B on a single GPU without LoRA; (ii) the per-step overhead stays minor, rollout generation dominating and the probe central difference not, which the controlled measurement of Appendix~\ref{sec:overhead} puts at $+7.9\%$ in the worst case, when the inner projection is active at essentially every step; (iii) the probe passes control the memory peak by micro-batch splitting and coexist with gradient checkpointing.
\end{remark}

\begin{remark}[per-step operation count]\label{app:remC1}
The ``single backward pass'' of Theorem~\ref{app:thmC2} names the final $\theta$-gradient readout, not the entire step. On the small $M_{\mathrm{fam}}n$-sequence probe (micro-batched at batch size $2$), one engaged regularizer step adds: $M_{\mathrm{fam}}$ hooked forward passes, each with a first-order activation gradient, to build the stop-gradient control matrix $\bar C$ (no \texttt{create\_graph}, hence flash-attention- and sharding-compatible); per family, two gated forward passes for the central difference ($\ell_m$ at $\1\pm\epsilon\hat W_{m,\cdot}$) plus one nominal forward for the $1/\ell_m$ normalizer; and one backward pass for the whole proxy $\theta$-gradient. No second-order graph is ever built. The near-hard projection re-runs this block up to $K_{\max}$ times per step, but only while $\Bshared>\tau$. On the main Qwen2.5-7B $\beta{=}0$ run the regularizer engages every step ($1500/1500$). Since the probe carries a handful of short sequences while a step is dominated by rollout generation, the same-GPU worst case is the $+7.9\%$ of Appendix~\ref{sec:overhead}.
\end{remark}

\section{CD-RFT Training Step}\label{sec:app_alg}

Algorithm~\ref{alg:cdrft} gives one training step of \CDRFT{} on top of a GRPO/PPO backbone, using the single-backward-pass first-order proxy of Section~\ref{sec:centraldiff}. Steps~2's two gated forward passes carry constant gate scalings independent of $\theta$, so the entire regularizer gradient is obtained by one backward pass and differs from $\nabla_\theta\Rproxy$ by $O(\epsilon^2)$ (Theorem~\ref{app:thmC2}); the stopping test in Step~4 uses a cheap first-order readout of $\Bshared$ on the same probe.

\begin{algorithm}[t]
\caption{One CD-RFT training step}
\label{alg:cdrft}
\begin{algorithmic}[1]
\Require policy $\theta$; probe set $\Pi$; weight $\lambda$; step $\epsilon$; target $\tau$; step sizes $\eta,\eta_{\mathrm{proj}}$; cap $K_{\max}$
\State \textbf{Backbone update:} sample rollouts, compute $\Lpost$ and $\nabla_\theta\Lpost$
\State set sublayer gates $g\!\leftarrow\!\1$ on $\Pi$
\For{each family $m$} \Comment{one ordinary backward, stop-grad $\theta$}
  \State $\ell_m\!\leftarrow\!\mathbb E_{\Pi_m}[\log\pi_\theta(y^\star)]$; skip if $|\ell_m|<\epsilon_0$
  \State $\bar C_{m,\cdot}\!\leftarrow\!(1/\ell_m)\,\partial_g \ell_m|_{g=\1}$
\EndFor
\State $\bar G\!\leftarrow\!\bar C\bar C^\top$;\ $\bar{\Rproxy}\!\leftarrow\!\operatorname{tr}(\bar G^2)/\operatorname{tr}(\bar G)^2$
\State $W\!\leftarrow\!\tfrac{4}{\operatorname{tr}(\bar G)^2}(\bar G\bar C-\bar{\Rproxy}\operatorname{tr}(\bar G)\bar C)$ \Comment{closed form, Lemma~\ref{app:lemC1}}
\State $\Lproxy\!\leftarrow\!0$
\For{each family $m$}
  \State $\hat W_m\!\leftarrow\!W_{m,\cdot}/\lVert W_{m,\cdot}\rVert$ \Comment{unit dir., Rmk.~\ref{app:propC4}}
  \State $\ell^{+}\!\leftarrow\!\ell_m(\theta;\1\!+\!\epsilon\hat W_m)$;\ $\ell^{-}\!\leftarrow\!\ell_m(\theta;\1\!-\!\epsilon\hat W_m)$
  \State $\Lproxy\!\leftarrow\!\Lproxy+\tfrac{\lVert W_{m,\cdot}\rVert}{\ell_m}\cdot\tfrac{\ell^{+}-\ell^{-}}{2\epsilon}$
\EndFor
\State $\nabla_\theta\Lproxy\!\leftarrow\!\textsc{Backward}(\lambda\Lproxy)$ \Comment{single pass, flash-safe}
\State $\theta\!\leftarrow\!\theta-\eta(\nabla_\theta\Lpost+\nabla_\theta\Lproxy)$
\State $k\!\leftarrow\!0$ \Comment{inner-loop projection, Def.~\ref{app:defC3}}
\While{$\Bshared(\Pi;\theta)>\tau$ \textbf{and} $k<K_{\max}$}
  \State recompute $\bar C,W,\Lproxy$;\ \ $\theta\!\leftarrow\!\theta-\eta_{\mathrm{proj}}\nabla_\theta\Lproxy$;\ \ $k\!\leftarrow\!k+1$
\EndWhile
\State \Return $\theta$
\end{algorithmic}
\end{algorithm}

\section{Full Experimental Settings}\label{sec:settings}

\paragraph{Training data.}
Three families of verifiable tasks in a balanced mixture, each resampled to 3{,}000 prompts (9{,}000 total), Table~\ref{tab:data}.

\begin{table}[H]
\centering

\small
\begin{tabular}{llcc}
\toprule
\addlinespace[1pt]
Family & Dataset & Pool & Sampled \\
\addlinespace[1pt]
\midrule
\addlinespace[3pt]
Math & DeepScaleR \citep{luo2025deepscaler} & 40{,}315 & 3{,}000 \\
Code & code-r1-12k \citep{liu2025coder1} & 12{,}458 & 3{,}000 \\
Logic & GURU logic \citep{cheng2025guru} & 3{,}126 & 3{,}000 \\
\bottomrule
\end{tabular}
\caption{Training data.}
\label{tab:data}

\end{table}

\paragraph{Evaluation protocol.}
Nine held-out benchmarks, fixed temperature applied identically to all methods, audited for leakage below; max generation length 4096 for mathematics and 2048 for code/logic; hard set = AIME24/AIME25/Minerva (Table~\ref{tab:eval}). The summarized $\passk$ of Table~\ref{tab:main} therefore uses a per-benchmark $k$: pass@256 for the math anchors, pass@16 for MATH500 and ordering, pass@64 for code and Logic-graph. The value is the largest $k$ each set supports at its sample count, subject to the set still discriminating between methods: on the easier benchmarks a large $k$ saturates every arm near the ceiling, which hides both the ordering between methods and the large-$k$ coverage degradation of the reward-maximizing baselines that Section~\ref{sec:capability} examines. AMC23 illustrates the saturated regime (pass@256 $\approx100$ for every method, Table~\ref{tab:ladder}); the hard set retains headroom at the same $k$ and is where the coverage comparison is read. Paired arms share initialization, data order, and schedule, so each capability comparison in Section~\ref{sec:capability} isolates the regularizer; the temperature sweep (Appendix~\ref{sec:app_temp}) and the checkpoint sweep (Appendix~\ref{sec:app_ckpt}) further test that reading.

\begin{table}[H]
\centering

\small
\begin{tabular}{llccc}
\toprule
\addlinespace[1pt]
Domain & Benchmark & \# & $n$ & Temp. \\
\addlinespace[1pt]
\midrule
\addlinespace[3pt]
Math & MATH500 & 500 & 16 & 0.8 \\
Math & AMC23 & 40 & 256 & 0.8 \\
Math & AIME24 & 30 & 256 & 0.8 \\
Math & AIME25 & 30 & 256 & 0.8 \\
Math & Minerva & 272 & 256 & 0.8 \\
Code & HumanEval+ & 164 & 64 & 0.7 \\
Code & MBPP & 500 & 64 & 0.7 \\
Logic & ordering-puzzle & 100 & 64 & 0.7 \\
Logic & Logic-graph & 100 & 64 & 0.7 \\
\bottomrule
\end{tabular}
\caption{Evaluation protocol (primary experiment, Qwen2.5-7B). Benchmarks: MATH500 \citep{hendrycks2021math}, Minerva \citep{lewkowycz2022minerva}, HumanEval+ \citep{chen2021evaluating,liu2023evalplus}, MBPP \citep{austin2021mbpp}, ordering-puzzle/Logic-graph \citep{cheng2025guru}. The Llama-3.2-3B evaluation (Appendix~\ref{sec:app_llama}) uses the same benchmarks at the training rollout temperature.}
\label{tab:eval}

\end{table}

\paragraph{Train/evaluation decontamination.}
We audit every held-out benchmark against every training pool at three levels: raw string identity; a normalized hash (lowercased, all non-alphanumeric characters stripped); and near-duplication by word $8$-gram Jaccard over an inverted index, reported at thresholds $.5$ and $.8$. For the code family we additionally match the executable content, comparing normalized assertion sets of the evaluation test suites against the training functional tests. Table~\ref{tab:decontam} gives the audit.

Eight of the nine benchmarks are clean at every level, with no exact, normalized, or near-duplicate hit and no shared assertion. The exception is MATH500, of which $7$ items ($1.4\%$) are normalized-exact matches of items in the DeepScaleR pool. Because each family is resampled to $3{,}000$ prompts from a much larger pool, pool membership overstates exposure: replaying the training shuffle shows that exactly $1$ of those $7$ items ($0.2\%$ of MATH500) was drawn into the data the runs actually saw. On that item the trained arms do not exceed the untrained base ($9/16$ correct for base, $8/16$ for GRPO, $9/16$ for \CDRFT{}), and on all $7$ the base and GRPO ($\beta{=}0$) $\passone$ are identical ($17.9$), so there is no memorization signature.

We nonetheless re-score MATH500 with all $7$ items removed, identically for every arm, as a conservative upper bound. Every arm gains between $+0.1$ and $+0.7$ points (the removed items are harder than average), and the paired \CDRFT{}$-$GRPO difference keeps its sign at $k{=}1,2,4,8$ on both axes; at $k{=}16$, where MATH500 is near saturation ($\approx90$ for all arms), the difference moves within $\pm0.4$ points of zero on both pairs. No conclusion in Section~\ref{sec:capability} depends on the contaminated items. The logic row deserves a separate note: the logic training pool and the logic benchmarks are drawn from the same GURU generator family, so their $8$-gram overlap ($\max J=.35$ on ordering-puzzle) measures shared templates, not shared instances; exact and normalized instance matches are zero, which is what leakage would require.

\begin{table}[H]
\centering

\small
\setlength{\tabcolsep}{4pt}
\begin{tabular}{llccccc}
\toprule
\addlinespace[1pt]
Benchmark & pool & exact & norm. & near & assert & $\max J$ \\
\addlinespace[1pt]
\midrule
\addlinespace[3pt]
MATH500        & math  & 1 & 7 & 8 & --- & 1.00 \\
AMC23          & math  & 0 & 0 & 0 & --- & .14 \\
AIME24         & math  & 0 & 0 & 0 & --- & .21 \\
AIME25         & math  & 0 & 0 & 0 & --- & .20 \\
Minerva        & math  & 0 & 0 & 0 & --- & .02 \\
GSM8K          & math  & 0 & 0 & 0 & --- & .06 \\
HumanEval+     & code  & 0 & 0 & 0 & 0 & .02 \\
MBPP           & code  & 0 & 0 & 0 & 0 & .07 \\
Logic-ordering & logic & 0 & 0 & 0 & --- & .35 \\
Logic-graph    & logic & 0 & 0 & 0 & --- & .01 \\
\bottomrule
\end{tabular}
\caption{Train/evaluation contamination audit. ``near'' = word $8$-gram Jaccard $\ge.8$; ``assert'' = shared normalized test assertions (code only). $\max J$ is the largest Jaccard against any training item.}
\label{tab:decontam}

\end{table}

\paragraph{Hyperparameters (primary experiment, Qwen2.5-7B).}
On the primary model all methods share the same recipe, differing only in the regularizer/KL (Table~\ref{tab:hparam}); the second-model Llama-3.2-3B setup and its differences are given in Appendix~\ref{sec:app_llama}.

\begin{table}[H]
\centering

\small
\begin{tabular}{ll}
\toprule
\addlinespace[1pt]
Item & Value \\
\addlinespace[1pt]
\midrule
\addlinespace[3pt]
Learning rate / schedule & $2\times10^{-6}$ / cosine \\
Max generation length & 512 \\
Samples per query $G$ & 4 \\
Global batch & 32 completions (8 prompts $\times$ 4) \\
Training steps & 1500 (report step-1000 at main table) \\
Random seed & 0 (training and rollout sampling) \\
KL coefficient $\beta$ & 0 or 0.01 (two paired axes) \\
CD-RFT weight $\lambda$ & 1 \\
CD-RFT probes $n$ & 3 / family \\
CD-RFT target concentration $\tau$ (\%) & 70 \\
CD-RFT inner-loop cap & 12 steps / training step \\
CD-RFT difference step $\epsilon$ & 0.05 \\
\bottomrule
\end{tabular}
\caption{Hyperparameters (primary experiment, Qwen2.5-7B). All arms share seed 0, so each paired comparison holds initialization and data order fixed.}
\label{tab:hparam}

\end{table}

\paragraph{Software and hardware.}
Sampling and evaluation use vLLM 0.23.0 with the rollout GPU-memory utilization set to 0.6; training, probing, and evaluation are each performed on a single GPU, on NVIDIA A100 (80GB) and NVIDIA RTX PRO 6000 (Blackwell) GPUs. Table~\ref{tab:stack} lists the software stack.

\begin{table}[H]
\centering

\small
\setlength{\tabcolsep}{6pt}
\begin{tabular}{ll}
\toprule
\addlinespace[1pt]
Component & Version \\
\addlinespace[1pt]
\midrule
\addlinespace[3pt]
OS / Python & Linux / 3.11 \\
PyTorch (CUDA) & 2.11 (CUDA 12), bf16 \\
Transformers & HuggingFace \\
RL trainer & TRL 1.6.0 (GRPO) \\
Sampling / eval & vLLM 0.23.0 \\
$\Bshared$ probe backend & eager attention (2nd-order) \\
\bottomrule
\end{tabular}
\caption{Software stack.}
\label{tab:stack}

\end{table}

\section{Additional Mechanistic Results}\label{sec:app_mech}

\subsection{Activation-Level Metrics}\label{sec:zhang}

Section~\ref{sec:mechanism} reports that two of the three activation-level metrics of \citet{zhang2025reinforcement} reproduce while the third does not; Tables~\ref{tab:zhang_actkurt} and~\ref{tab:zhang_info} give the underlying numbers. All three are computed from EAP edge attributions \citep{syed2024attribution} on the both-correct subset of the base model and the method under test, the standard attribution-difference protocol. Because each pairing induces a slightly different both-correct subset, and the Base row shown is the subset baseline of the $\beta{=}0$ pairing, the tables are read within a column.

\begin{table}[H]
\centering

\small
\setlength{\tabcolsep}{5pt}
\begin{tabular}{lcccccc}
\toprule
\addlinespace[1pt]
& \multicolumn{3}{c}{Act.\ Intensity $\uparrow$} & \multicolumn{3}{c}{Dist.\ Kurtosis $\downarrow$} \\
\cmidrule(lr){2-4}\cmidrule(lr){5-7}
Method & MATH500 & MBPP & graph & MATH500 & MBPP & graph \\
\addlinespace[1pt]
\midrule
\addlinespace[3pt]
Base & .00098 & .00157 & .00159 & 359 & 542 & 277 \\
GRPO ($\beta{=}0$) & .00170 & .00310 & .00263 & 198 & 479 & 133 \\
\textbf{CD-RFT ($\beta{=}0$)} & .00137 & .00165 & .00191 & 206 & 499 & 197 \\
GRPO ($\beta{=}0.01$) & .00146 & .00260 & .00379 & 178 & 424 & 191 \\
\textbf{CD-RFT ($\beta{=}0.01$)} & .00147 & .00273 & .00297 & 204 & 433 & 255 \\
\bottomrule
\end{tabular}
\caption{The two activation-level signatures of \citet{zhang2025reinforcement} that reproduce robustly: activation intensity rises and distribution kurtosis falls, for every method on every domain (12/12 each), relative to the base model. graph = Logic-graph.}
\label{tab:zhang_actkurt}

\end{table}

Activation intensity rises and kurtosis falls for every method on every domain (Table~\ref{tab:zhang_actkurt}). The signature that reinforcement learning raises circuit usage and flattens its magnitude distribution therefore reproduces not only for vanilla GRPO but for all four trained models, and forms the starting point on which the analysis builds.

\begin{table}[H]
\centering

\small
\setlength{\tabcolsep}{6pt}
\begin{tabular}{lccc}
\toprule
\addlinespace[1pt]
Method & MATH500 & MBPP & graph \\
\addlinespace[1pt]
\midrule
\addlinespace[3pt]
Base & .271$\pm$.027 & .255$\pm$.015 & .491$\pm$.024 \\
GRPO ($\beta{=}0$) & .008$\pm$.030 & .015$\pm$.004 & .496$\pm$.057 \\
\textbf{CD-RFT ($\beta{=}0$)} & \textbf{.219$\pm$.048} & \textbf{.270$\pm$.019} & \textbf{.568$\pm$.037} \\
GRPO ($\beta{=}0.01$) & .142$\pm$.036 & .020$\pm$.006 & .009$\pm$.095 \\
\textbf{CD-RFT ($\beta{=}0.01$)} & .036$\pm$.058 & .011$\pm$.006 & .080$\pm$.097 \\
\bottomrule
\end{tabular}
\caption{Information complexity, the direction-unstable third metric. Bootstrap standard deviations ($B{=}2000$) on the both-correct subset of each pairing ($n{=}73$--$100$ on MATH500/MBPP, $n{=}19$--$23$ on Logic-graph). graph = Logic-graph, which carries the logic-domain evidence.}
\label{tab:zhang_info}

\end{table}

Information complexity (Table~\ref{tab:zhang_info}) does not behave this way. On MATH500 and MBPP it stays high under \CDRFT{} ($\beta{=}0$) while collapsing to near zero under the matched GRPO ($\beta{=}0$), .219$\pm$.048 against .008$\pm$.030 and .270$\pm$.019 against .015$\pm$.004, gaps of several bootstrap standard deviations; the same ordering reappears on Logic-graph (.568$\pm$.037 against .496$\pm$.057). Yet the metric does not track capability: GRPO ($\beta{=}0.01$) is stronger than the base on every domain of Table~\ref{tab:main}, while its information complexity on MBPP falls from .255 to .020, a drop of more than $90\%$. A metric that splits this way between two equally healthy models cannot on its own say what reinforcement learning changed, which is what motivates moving to the control axis. The well-powered form of this reading is the $\beta{=}0$ pairing, where both arms sit high enough on mathematics and code to separate.

\subsection{Random Reference for $\Bshared$}\label{sec:app_mc}

The lower bound $1/M$ of Proposition~\ref{prop:bounds} is attained only when the rows are exactly orthogonal with equal norms, which random rows in finite dimension never are, so $1/M$ is a bound and not the value a random control matrix would produce. We therefore estimate the random reference by simulation at the shape used throughout ($M{=}3$ families, $K{=}56$ sublayer gates): drawing i.i.d.\ Gaussian rows over $50{,}000$ trials gives
\[
\Bshared^{\mathrm{rand}}=42.8\pm3.4,\qquad p_{95}=49.0,\quad p_{99}=51.9,
\]
against the bound $1/M=33.3$. This reference depends only on $(M,K)$, hence is shared by every arm and every model in Figure~\ref{fig:crossmodel}. All five base models of Figure~\ref{fig:crossmodel} exceed the $p_{99}$ of the random reference: Qwen2.5-3B $59.2$, Qwen2.5-1.5B $64.3$, Llama-3.1-8B $71.3$, Llama-3.2-3B $80.5$, Qwen2.5-7B $83.4$, i.e.\ between $+16$ and $+41$ above it. These are the exact $\Bshared(\Cmat)$ values of that figure, on the same protocol as the Base row of Table~\ref{tab:main}. The cross-model reading of Section~\ref{sec:mechanism} therefore holds against a simulated random reference and not merely against the bound.

$\Bshared$ also responds to row-norm heterogeneity, not to direction alone. Proposition~\ref{prop:bounds}(iv) gives invariance under global scaling $\Cmat\mapsto c\Cmat$ and under orthogonal gate reparameterization $\Cmat\mapsto\Cmat Q$, but not under \emph{per-row} rescaling, and each row of $\Cmat$ carries its own normalizer $1/\ell_m$, whose scale differs across families with target length. This is why every comparison we report is made at a fixed protocol and, for the method claim, as a same-seed same-probe paired difference (Figure~\ref{fig:mechanism}), both of which hold the row-norm structure of the estimator fixed; the cross-model panel, whose arms do not share a protocol, is read only by ordering.

\subsection{Direction Versus Row Norm}\label{sec:app_rownorm}

Holding the protocol fixed still leaves a question the paired difference cannot answer on its own, so we separate the two contributions directly. Writing $d_m=\lVert\Cmat_{m,\cdot}\rVert_2^2$, $D=\operatorname{diag}(d)$ and $R$ for the matrix of row cosines, the Gram matrix factors as $G_{\Cmat}=D^{1/2}RD^{1/2}$, and we report alongside $\Bshared$ two quantities computed from the same $\Cmat$:
\begin{equation}\label{eq:bdir}
\Bdir(\Cmat)=\frac{\lambda_{\max}(R)}{M},
\qquad
\Bnorm(\Cmat)=\max_m \frac{d_m}{\sum_{m'} d_{m'}}.
\end{equation}
$\Bdir$ normalizes each row to unit length first, so it is invariant to per-row rescaling and measures direction sharing alone; it keeps the range $[1/M,1]$, attaining $1/M$ exactly when the rows are pairwise orthogonal and $1$ when they are collinear. $\Bnorm$ is what $\Bshared$ would equal if the rows were exactly orthogonal, i.e.\ the row-norm contribution alone. The two are not additive, so we plot them side by side rather than as a decomposition of $\Bshared$.

\begin{figure}[tp]
\centering
\includegraphics[width=\linewidth]{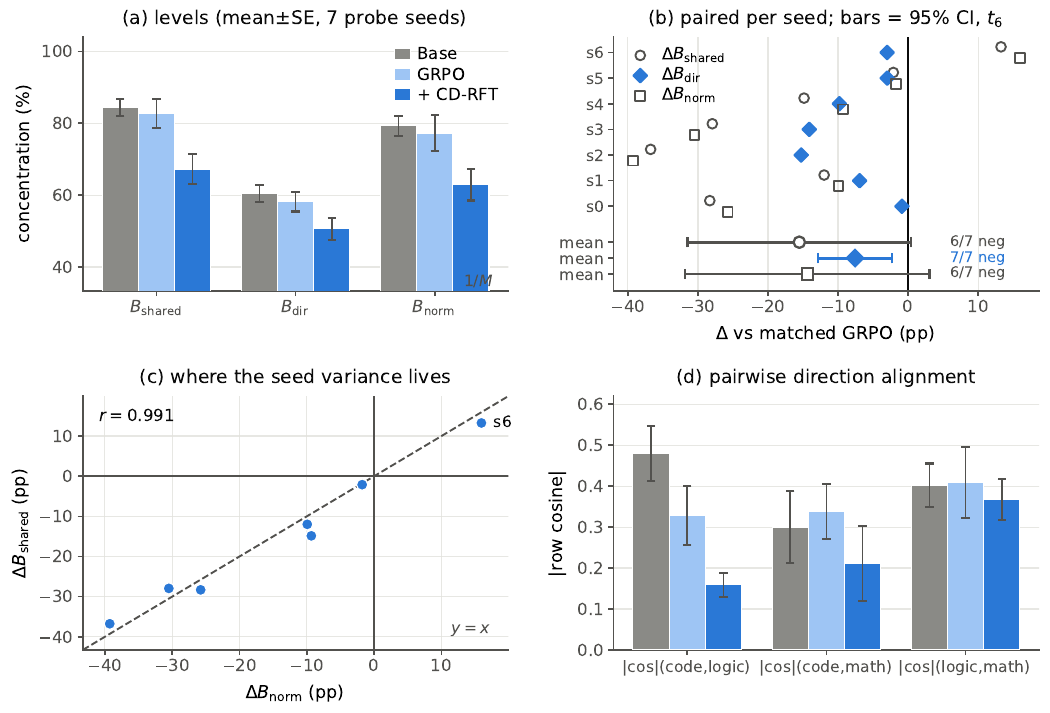}
\caption{\CDRFT{} lowers direction sharing. \textbf{(a)} All three measures at the $\beta{=}0$ arms; $\Bdir$ sits far below $\Bshared$ at every arm, so $\Bshared$ overstates how collinear the control directions are. \textbf{(b)} Paired per seed against the matched GRPO: $\Delta\Bdir=-7.6\pm2.1$ is negative on $7/7$ seeds and its $95\%$ interval excludes zero, while the interval on $\Delta\Bshared$ does not. \textbf{(c)} The seed-to-seed variation of $\Delta\Bshared$ lies on the diagonal against $\Delta\Bnorm$ ($r{=}0.99$): the row-norm term is what injects it, and the one seed that flips $\Delta\Bshared$ positive (s6) still has $\Delta\Bdir<0$. \textbf{(d)} Pairwise direction alignment in magnitude; the code--logic cosine falls from $0.48$ to $0.16$, and \CDRFT{} is the lowest on all three pairs.}
\label{fig:rownorm}
\end{figure}

Figure~\ref{fig:rownorm} reports both on the $\beta{=}0$ pair, over the seven probe seeds of Table~\ref{tab:main}. Direction sharing falls under \CDRFT{}: the paired $\Delta\Bdir$ is $-7.6\pm2.1$ points and negative on every seed, against $6/7$ for $\Delta\Bshared$, and its $95\%$ interval (Student $t$, six degrees of freedom) excludes zero where the interval on $\Delta\Bshared$ does not. Removing the row norms therefore sharpens the effect rather than dissolving it, because the row-norm term carries most of the seed-to-seed variance: $\Delta\Bshared$ tracks $\Delta\Bnorm$ across seeds at $r=0.99$, and the single seed on which $\Delta\Bshared$ turns positive is one where $\Delta\Bdir$ remains negative. The row cosines move with it: in magnitude \CDRFT{} is the lowest of the three arms on every family pair, the code--logic pair dropping from $0.48$ to $0.16$. Magnitudes are the relevant summary because a signed average over seeds cancels, and anti-alignment raises $\lambda_{\max}(R)$ just as alignment does.

Because $\lVert\Cmat_{m,\cdot}\rVert_2=\lVert\partial_g\ell_m|_{g=\1}\rVert_2/\lvert\ell_m\rvert$, a regularizer could in principle lower $\Bshared$ by inflating $\lvert\ell_m\rvert$ on the family that dominates the energy budget, which would work against the reward flux rather than with it. This is not what happens. On the family holding the largest share at baseline, $\lvert\ell_m\rvert$ rises by $12.5\%$ from the matched GRPO while its row norm falls by $63\%$, so the implied gate sensitivity $\lVert\partial_g\ell_m\rVert_2$ falls by $59\%$: the normalizer accounts for about a fifth of the change and the rest is a genuine reduction in how strongly that family's reward flux responds to the gates.

\subsection{Paired Verification and Cross-Model Generality}

Absolute $\Bshared(\Cmat)$ depends on which probe samples are drawn, so Figure~\ref{fig:mechanism} re-reads the measurement of Table~\ref{tab:main} as a same-seed, same-probe paired difference, which cancels the shared probe noise. Figure~\ref{fig:crossmodel} then checks that the bottleneck is not an artifact of the single model we intervene on, and Figure~\ref{fig:edgepr} repeats the distributedness check of Figure~\ref{fig:proxydist}(b) at the finer EAP edge granularity. The participation ratio of a control row is $\mathrm{PR}(C_{m,\cdot})=\big(\sum_k|C_{m,k}|\big)^2\big/\sum_k C_{m,k}^2$, the effective number of gates it spreads over: $\mathrm{PR}=1$ when a single gate carries the row and $\mathrm{PR}=K$ when all gates carry it equally.

\begin{figure}[tp]
\centering
\includegraphics[width=\textwidth]{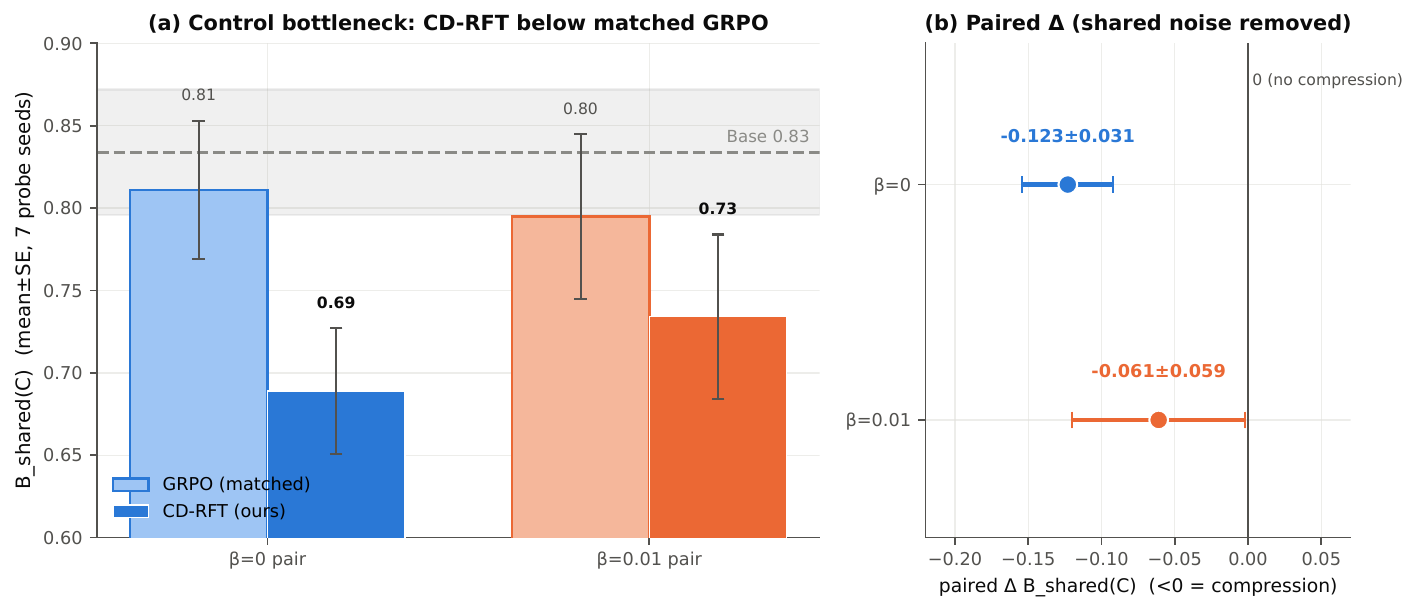}
\caption{Forest plot of the paired difference $\Delta = \Bshared(\Cmat)^{\text{CD-RFT}} - \Bshared(\Cmat)^{\text{GRPO}}$ under the same seed and probes. \CDRFT{} ($\beta{=}0$) shows a robust negative shift ($\Delta = -12.3\pm3.1$, negative in 6 of 7 seeds); \CDRFT{} ($\beta{=}0.01$) is consistent in direction but smaller.}
\label{fig:mechanism}
\end{figure}

\begin{figure}[tp]
\centering
\includegraphics[width=\textwidth]{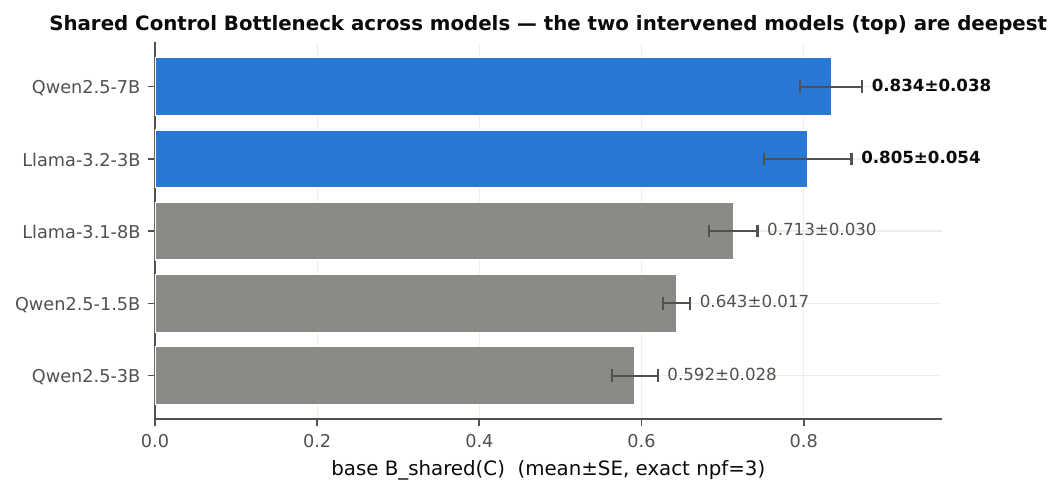}
\caption{$\Bshared(\Cmat)$ for five base models under the same multi-task probe and exact protocol. All lie above the $99$th percentile of the simulated random reference of Appendix~\ref{sec:app_mc}; the two models we train \CDRFT{} on (Qwen2.5-7B and Llama-3.2-3B, in blue) have the deepest bottlenecks, and hence the most decoupling headroom.}
\label{fig:crossmodel}
\end{figure}

\begin{figure}[tp]
\centering
\includegraphics[width=\textwidth]{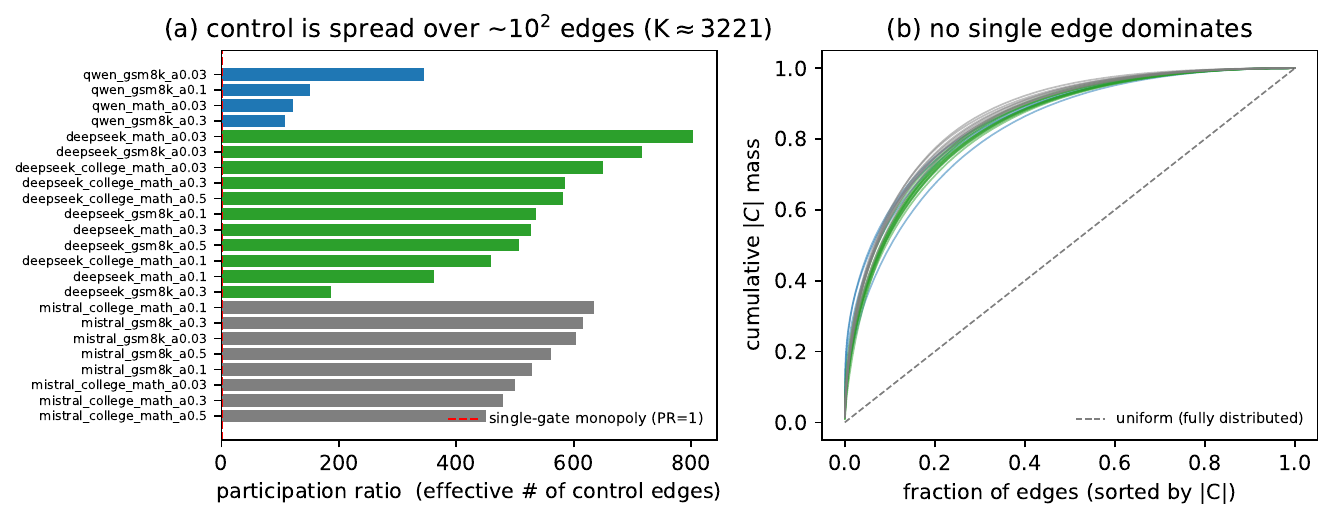}
\caption{Distributedness of Figure~\ref{fig:proxydist}(a) at the finer EAP edge granularity ($K\approx3221$) on the offline Part-I control vectors. \textbf{(a)} Participation ratio and \textbf{(b)} cumulative $|C|$ mass give the same no-monopoly conclusion; the ratio is larger here only because the edge decomposition is finer, and aggregation into sublayers is non-orthogonal, so the sublayer figure reports the on-granularity value.}
\label{fig:edgepr}
\end{figure}

\section{Additional Capability Results}\label{sec:app_cap}

\subsection{Full Per-Benchmark \passk{} Ladder}

Table~\ref{tab:ladder} gives the per-benchmark results behind the domain summary of Table~\ref{tab:main} (step-1000, five matched methods). The coverage column reports the highest $k$ of each set: pass@256 for the math anchors, pass@16 for MATH500 and ordering, pass@64 for code and Logic-graph. AMC23 (40 problems) is near-saturated there (pass@256 $\approx 100$ for every method), which treats all methods alike; the hard-set coverage advantage is carried by AIME24/AIME25/Minerva.

\begin{table}[H]
\centering

\small
\setlength{\tabcolsep}{8pt}
\begin{tabular}{llcccccc}
\toprule
\addlinespace[1pt]
Domain & Bench (metric) & Base & G($\beta0$) & \textbf{CD($\beta0$)} & G($\beta.01$) & \textbf{CD($\beta.01$)} \\
\addlinespace[1pt]
\midrule
\addlinespace[3pt]
math & MATH500 p@1 & 48.5 & 57.5 & \textbf{57.7} & 53.7 & 54.3 \\
& MATH500 p@16 & 89.0 & 89.4 & 89.4 & 89.4 & \textbf{90.0} \\
math & AMC23 p@1 & 30.9 & 35.7 & \textbf{37.3} & 33.1 & 34.9 \\
& AMC23 p@256 & 97.5 & 100.0 & \textbf{100.0} & 97.5 & \textbf{100.0} \\
math & AIME24 p@1 & 5.9 & 7.2 & \textbf{7.2} & 5.6 & 5.9 \\
& AIME24 p@256 & 63.3 & 50.0 & \textbf{63.3} & 60.0 & \textbf{63.3} \\
math & AIME25 p@1 & 3.9 & 4.0 & 4.1 & 3.6 & \textbf{4.1} \\
& AIME25 p@256 & 43.3 & 56.7 & 53.3 & 43.3 & \textbf{63.3} \\
math & Minerva p@1 & 10.7 & 13.0 & \textbf{13.6} & 11.7 & 12.2 \\
& Minerva p@256 & 63.2 & 64.3 & \textbf{65.1} & 62.5 & 63.6 \\
code & HE+ p@1 & 66.7 & 68.4 & \textbf{70.5} & 69.7 & 70.3 \\
& HE+ p@64 & 96.3 & 97.0 & \textbf{98.2} & 95.7 & 97.0 \\
code & MBPP p@1 & 42.3 & 47.6 & \textbf{50.6} & 48.4 & 49.6 \\
& MBPP p@64 & 87.8 & 88.8 & 88.6 & \textbf{89.8} & 89.2 \\
logic & ordering p@1 & 47.1 & 54.3 & \textbf{54.8} & 51.3 & 52.0 \\
& ordering p@16 & 94.6 & \textbf{96.2} & 95.7 & 94.2 & 96.1 \\
logic & Logic-graph p@1 & 19.3 & 25.0 & \textbf{26.1} & 21.7 & 20.8 \\
& Logic-graph p@64 & 87.0 & 92.0 & \textbf{94.0} & 90.0 & 91.0 \\
\bottomrule
\end{tabular}
\caption{Per-benchmark $\passk$ ladder underlying Table~\ref{tab:main}. All entries are percentages.}
\label{tab:ladder}

\end{table}

\subsection{Per-Benchmark Paired Gains and Coverage}

Figure~\ref{fig:gains} resolves the domain summary of Table~\ref{tab:main} into the per-benchmark paired gain of \CDRFT{} over its matched GRPO baseline, and Figure~\ref{fig:passk} traces how the coverage advantage of the with-KL variant grows with $k$ on the hard sets.

\begin{figure}[tp]
\centering
\includegraphics[width=\textwidth]{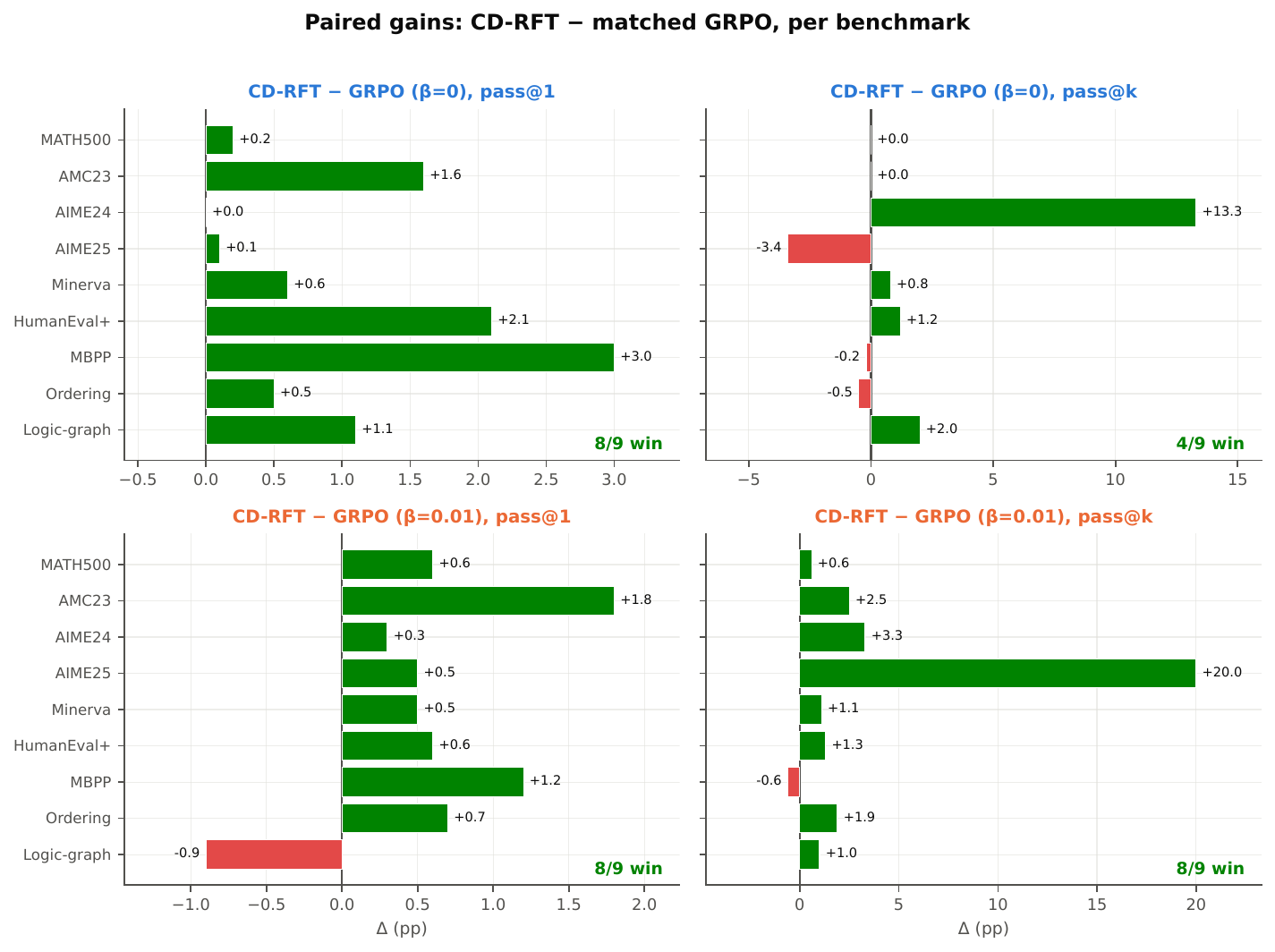}
\caption{Per-benchmark paired gains, \CDRFT{} minus its matched GRPO baseline, across all nine benchmarks in four panels ($\beta{=}0$ / $\beta{=}0.01$ crossed with $\passone$ / $\passk$). The $\beta{=}0$ axis leads on $\passone$ (8/9); the $\beta{=}0.01$ axis leads on $\passk$ coverage (8/9).}
\label{fig:gains}
\end{figure}

\begin{figure}[tp]
\centering
\includegraphics[width=\textwidth]{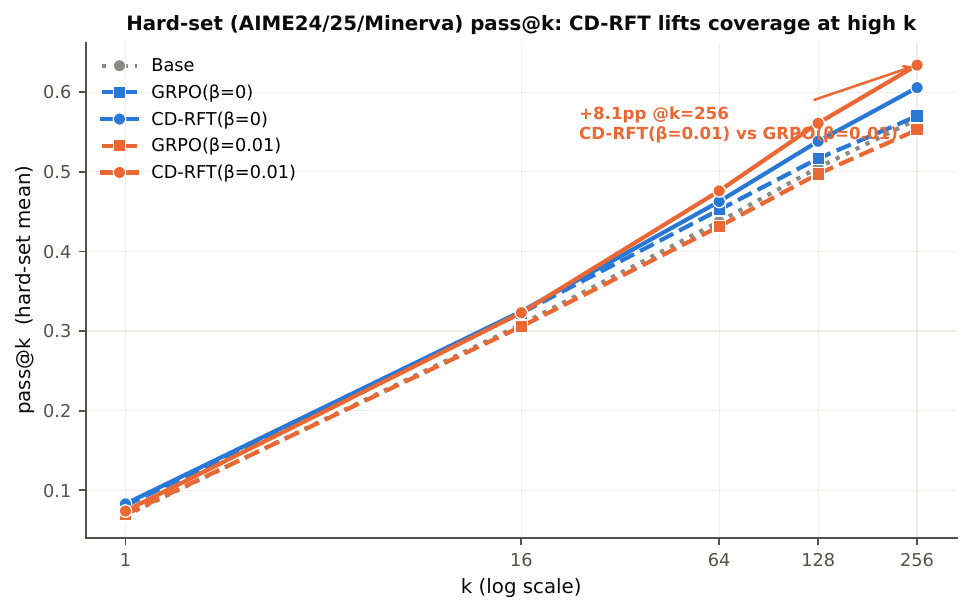}
\caption{Hard-set (AIME24/AIME25/Minerva) mean $\passk$ as a function of $k$ (log scale). \CDRFT{} ($\beta{=}0.01$) widens its coverage advantage over GRPO ($\beta{=}0.01$) at large $k$, reaching +8.1pp at pass@256.}
\label{fig:passk}
\end{figure}

\section{Full Temperature-Sweep Results}\label{sec:app_temp}

We repeat the evaluation at four temperatures (the sweep of each domain including its main-table working temperature) and count the fraction of paired $\times$ temperature cells (8 in total) in which ``\CDRFT{} $\geq$ matched RL'' holds (Table~\ref{tab:temp}; the trend is in Figure~\ref{fig:tempscan}).

\begin{table}[H]
\centering

\small
\begin{tabular}{lcc}
\toprule
\addlinespace[1pt]
Domain & pass@1 & pass@k \\
\addlinespace[1pt]
\midrule
\addlinespace[3pt]
Math hard set & \textbf{8/8} & 7/8 \\
Code & \textbf{8/8} & 7/8 \\
Logic & 3/8 & 5/8 \\
\bottomrule
\end{tabular}
\caption{Temperature robustness: paired-ordering win rate (\CDRFT{} $\geq$ matched RL, out of 8).}
\label{tab:temp}

\end{table}

On the math hard set and on code, the $\passone$ advantage of \CDRFT{} holds across all four temperatures (8/8 each) and $\passk$ holds at 7/8. Lowering the temperature raises $\passone$ and raising it raises $\passk$, yet the method ordering does not flip. The logic domain shows lower ordering consistency ($\passone$ 3/8), in line with the mild gain margin of this domain in the capability comparison, while the featured configuration \CDRFT{} ($\beta{=}0$) shows no collapse on any domain.

\begin{figure}[tp]
\centering
\includegraphics[width=\textwidth]{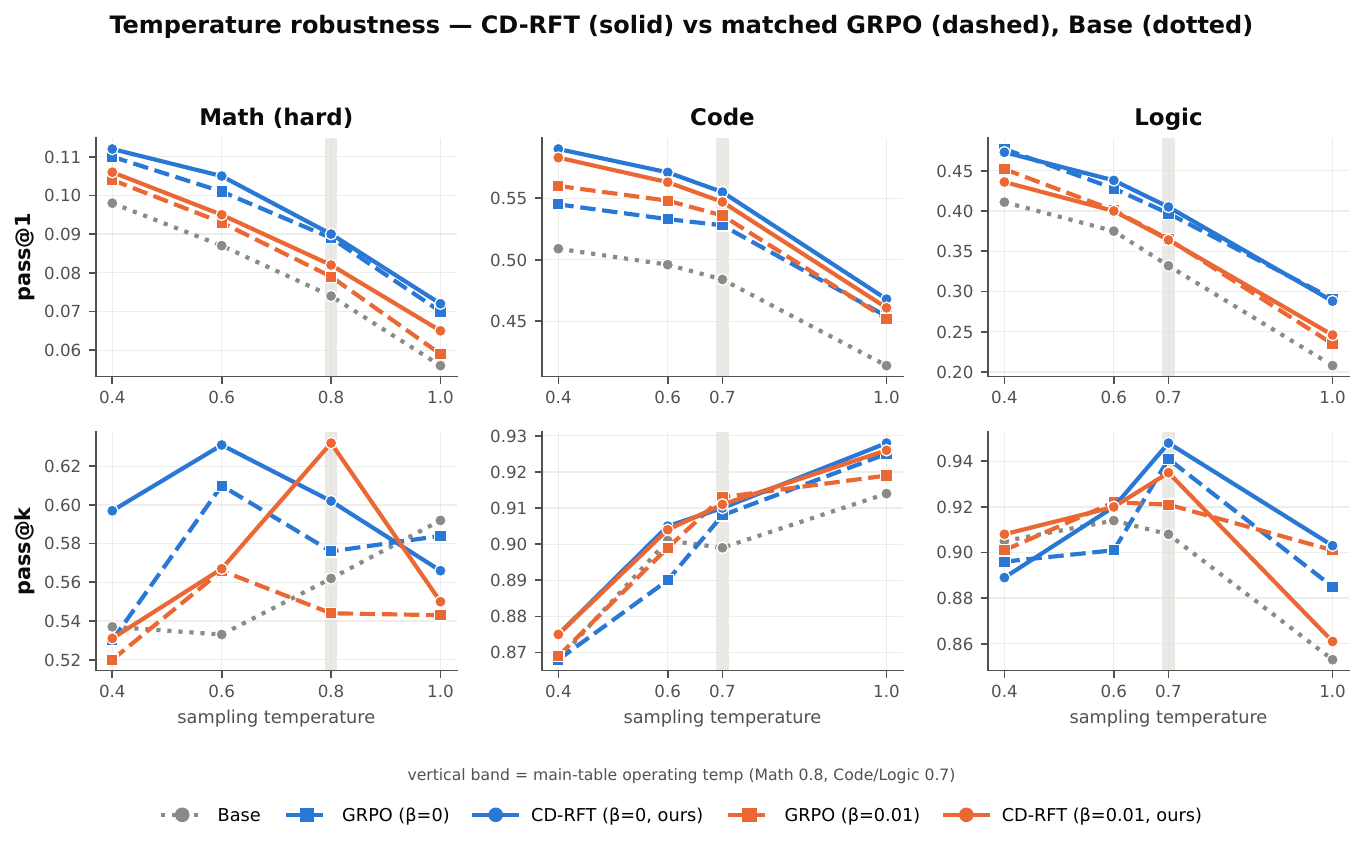}
\caption{Sampling-temperature sweep. Lowering the temperature raises $\passone$ and raising it raises $\passk$; the method ordering on the math hard set and code does not flip across the four temperatures.}
\label{fig:tempscan}
\end{figure}

\section{Proxy-Target Sensitivity}\label{sec:app_target}
We sweep the target concentration $\tau\in\{55,65,70\}$ with every other setting fixed at the main recipe (Qwen2.5-7B, step-1000, matched protocol), so $\tau$ is the sole variable. Table~\ref{tab:target} gives $\passone$ and the coverage $\passk$ for each benchmark.

The method is insensitive to $\tau$ across a broad band: the two looser settings, $70$ and $65$, lie within a small margin of each other on every benchmark and simply exchange which axis they favour. The main working point $70$ takes $\passone$ on the discriminative benchmarks (MATH500, AMC23, code, Logic-graph), while the slightly tighter $65$ trades that small deficit for large-$k$ coverage on the harder or more open-ended sets (MATH500, Minerva, code, ordering). Only over-compression, $\tau{=}55$, breaks the pattern: it trails on both axes across essentially all discriminative benchmarks, marking a floor below which tightening the constraint costs single-shot accuracy without buying coverage. $\tau$ should therefore be kept in the loose band, its exact value chosen by whether the deployment weights single-shot accuracy ($70$) or coverage ($65$).

\begin{table}[H]
\centering

\small
\setlength{\tabcolsep}{4pt}
\begin{tabular}{lcc}
\toprule
\addlinespace[1pt]
Benchmark & $\passone$ ($\tau{=}55/65/70$) & $\passk$ ($\tau{=}55/65/70$) \\
\addlinespace[1pt]
\midrule
\addlinespace[3pt]
MATH500     & 54.7 / 56.4 / \textbf{57.7} & 88.4 / \textbf{90.4} / 89.4 \\
AMC23       & 33.7 / 34.9 / \textbf{37.3} & \textbf{100.0} / \textbf{100.0} / \textbf{100.0} \\
AIME24      & 6.3 / 5.8 / \textbf{7.2} & \textbf{63.3} / 60.0 / \textbf{63.3} \\
AIME25      & 3.7 / 3.5 / \textbf{4.1} & \textbf{60.0} / 53.3 / 53.3 \\
Minerva     & 12.6 / \textbf{14.5} / 13.6 & 64.3 / \textbf{66.2} / 65.1 \\
code        & 49.7 / 51.5 / \textbf{55.5} & 90.2 / \textbf{91.7} / 91.0 \\
ordering    & 26.9 / \textbf{28.7} / 28.7 & 71.0 / \textbf{73.0} / 70.0 \\
Logic-graph & 23.2 / 26.1 / \textbf{26.1} & 89.0 / 90.0 / \textbf{94.0} \\
\bottomrule
\end{tabular}
\caption{Proxy-target sweep on Qwen2.5-7B ($\passone$ / coverage $\passk$, for $\tau{=}55/65/70$). Coverage $k$ is the highest sampled per set (MATH500/ordering p@16 and p@64 respectively, math anchors p@256, code/Logic-graph p@64). Bold marks the best of the three. Logic-graph uses the mid-difficulty split, matching the main table. All entries are percentages.}
\label{tab:target}

\end{table}

\section{Cross-Checkpoint Robustness}\label{sec:app_ckpt}

The main comparison reports one checkpoint (step-1000), fixed under the evaluation compute budget before any benchmark was scored and applied identically to all five arms; it is not the final step of training. Figure~\ref{fig:ckpt} sweeps the $\beta{=}0$ pair across checkpoints: \CDRFT{} ($\beta{=}0$) sits at or above its matched GRPO at every checkpoint on both axes.

Table~\ref{tab:final1500} gives the per-benchmark ladder at the \emph{final} step-1500 checkpoint under the full main-table protocol, the same benchmarks, sample counts, and temperatures as Table~\ref{tab:main}, applied identically to both arms, with the untrained base as an anchor.

\begin{table}[H]
\centering

\small
\setlength{\tabcolsep}{8pt}
\begin{tabular}{llccc}
\toprule
\addlinespace[1pt]
Domain & Bench (metric) & Base & G($\beta0$) & \textbf{CD($\beta0$)} \\
\addlinespace[1pt]
\midrule
\addlinespace[3pt]
math & MATH500 p@1 & 48.5 & 57.3 & \textbf{57.4} \\
& MATH500 p@16 & 89.0 & \textbf{92.0} & 90.4 \\
math & AMC23 p@1 & 30.9 & 35.6 & \textbf{37.3} \\
& AMC23 p@256 & 97.5 & 97.5 & \textbf{100.0} \\
math & AIME24 p@1 & 5.9 & 6.6 & \textbf{7.1} \\
& AIME24 p@256 & 63.3 & \textbf{66.7} & 56.7 \\
math & AIME25 p@1 & 3.9 & \textbf{4.2} & 4.0 \\
& AIME25 p@256 & 43.3 & 46.7 & \textbf{53.3} \\
math & Minerva p@1 & 10.7 & 12.9 & \textbf{13.6} \\
& Minerva p@256 & 63.2 & 64.0 & \textbf{66.5} \\
code & HE+ p@1 & 66.7 & 69.3 & \textbf{69.3} \\
& HE+ p@64 & 96.3 & 97.0 & \textbf{97.6} \\
code & MBPP p@1 & 42.3 & 47.7 & \textbf{50.5} \\
& MBPP p@64 & 87.8 & \textbf{88.8} & \textbf{88.8} \\
logic & ordering p@1 & 47.1 & 55.2 & \textbf{55.4} \\
& ordering p@16 & 94.6 & \textbf{95.5} & 94.2 \\
logic & Logic-graph p@1 & 19.3 & \textbf{24.9} & 24.5 \\
& Logic-graph p@64 & 87.0 & 88.0 & \textbf{90.0} \\
\addlinespace[1pt]
\midrule
\addlinespace[3pt]
\multicolumn{2}{l}{overall p@1} & 35.9 & 40.6 & \textbf{41.2} \\
\multicolumn{2}{l}{overall \passk{}} & 85.6 & 86.8 & \textbf{87.2} \\
\bottomrule
\end{tabular}
\caption{Per-benchmark ladder at the final step-1500 checkpoint, $\beta{=}0$ pair, full main-table protocol (cf.\ Table~\ref{tab:ladder} at step-1000). G = GRPO, CD = \CDRFT{}; Base is untrained and step-independent. Bold marks the row-wise best. All entries are percentages.}
\label{tab:final1500}

\end{table}

\begin{figure}[tp]
\centering
\includegraphics[width=\textwidth]{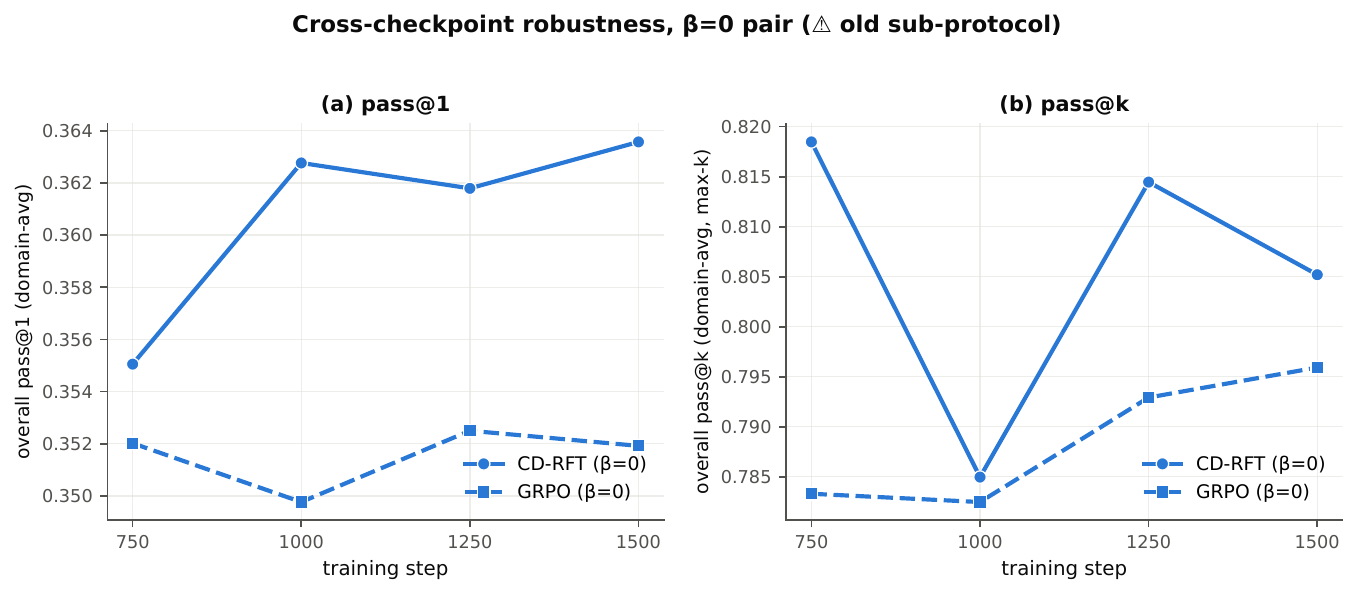}
\caption{Cross-checkpoint robustness for the $\beta{=}0$ pair: overall $\passone$ and $\passk$ versus training step. \CDRFT{} ($\beta{=}0$) stays at or above GRPO ($\beta{=}0$) at every checkpoint. The mathematics domain is scored here on subsampled MATH500 and Minerva, so the figure is read by ordering; Table~\ref{tab:final1500} gives the final checkpoint under the full main-table protocol.}
\label{fig:ckpt}
\end{figure}

\section{Second Model: Llama-3.2-3B}\label{sec:app_llama}

This appendix details the second-model transfer experiment of Section~\ref{sec:robustness}: the setup differences from the primary Qwen2.5-7B run, the full per-benchmark $\passk$ ladders, and the sampling-temperature behaviour.

\paragraph{Setup.} The method, the training data (the balanced three-family mixture of Table~\ref{tab:data}), and the benchmarks are shared with the primary experiment. The recipe is adapted for the weaker base model, and every adaptation is applied identically to the two trained arms, so the GRPO\,$\to$\,\CDRFT{} comparison stays controlled; Table~\ref{tab:llama_hp} lists the differences.

\textbf{Evaluation protocol.} The trained arms are evaluated few-shot ($K{=}2$, $K{=}4$ for code) with the same demonstrations, because the plain Llama-3.2-3B template otherwise holds them below the ability they have already reached. The \emph{Base} column is instead the \emph{zero-shot} untrained model, which sits at the reward floor and degenerates into format-violating and repeated generations---the starting point from which reinforcement learning departs. Since the two protocols differ, \textbf{the paired $\Delta$ is computed only between GRPO and \CDRFT{}; Base serves as an untrained-floor reference and is not differenced against them}.

\textbf{On the target $\tau$.} The value $\tau{=}50$ was fixed before this line was evaluated, neither transferred from the primary model nor tuned against benchmark scores, and it enters only the \CDRFT{} arm; by the proxy-target sweep of Section~\ref{sec:robustness} the comparison does not rest on its exact value.

\begin{table}[H]
\centering

\small
\begin{tabular}{lll}
\toprule
\addlinespace[1pt]
Item & Qwen2.5-7B & Llama-3.2-3B \\
\addlinespace[1pt]
\midrule
\addlinespace[3pt]
KL $\beta$ & 0 and 0.01 & 0 only \\
Prompt template & plain (\texttt{\textbackslash boxed}) & R1 think/answer \\
Few-shot & 0 & 2 (code 4) \\
Rollout temperature & 0.6 & 0.5 \\
Max generation length & 512 & 1024 \\
Target concentration $\tau$ (\%) & 70 & 50 \\
Reported checkpoint & step-1000 & step-1122 \\
\bottomrule
\end{tabular}
\caption{Llama-3.2-3B setup, as differences from the primary Qwen2.5-7B recipe (Table~\ref{tab:hparam}). All other knobs ($\lambda{=}1$, $\epsilon{=}0.05$, inner-loop cap $12$, learning rate/schedule) are unchanged and shared across arms.}
\label{tab:llama_hp}

\end{table}

\paragraph{Mechanism: control-sharing bottleneck.}
Table~\ref{tab:llama_bshared} reports the exact $\Bshared(\Cmat)$ for the three arms under the same protocol as the primary model (Section~\ref{sec:mechanism}; npf$=3$, second-order double-backward, five probe seeds paired to cancel shared-probe noise), using the Llama-3.2-3B tokenizer to build the probe. The signature reproduces: base $>$ GRPO $>$ \CDRFT{}, i.e.\ reinforcement learning lowers control sharing and the regularizer lowers it further, with the same-seed paired $\CDRFT{-}$GRPO $=-5.5\pm1.9$ ($2.9\sigma$, $4/5$ seeds negative).

\begin{table}[H]
\centering

\small
\setlength{\tabcolsep}{6pt}
\begin{tabular}{lc@{\hskip 16pt}lc}
\toprule
\addlinespace[1pt]
Arm & $\Bshared(\Cmat)\downarrow$ & Paired $\Delta$ & value \\
\addlinespace[1pt]
\midrule
\addlinespace[3pt]
Base           & $80.5\pm5.4$ & \CDRFT{}$-$GRPO & $-5.5\pm1.9$ \\
GRPO           & $77.0\pm4.5$ & \CDRFT{}$-$Base & $-9.0\pm3.8$ \\
\textbf{\CDRFT} & $\mathbf{71.5\pm3.5}$ & GRPO$-$Base & $-3.4\pm2.4$ \\
\bottomrule
\end{tabular}
\caption{Llama-3.2-3B exact $\Bshared(\Cmat)$ (step-1122, five probe seeds, mean$\pm$SE), same protocol as Table~\ref{tab:main}. Same-seed paired differences at right.}
\label{tab:llama_bshared}

\end{table}

\paragraph{Full $\passk$ ladder.}
Table~\ref{tab:llama_ladder} gives the complete ladder at the training rollout temperature, with all three arms evaluated few-shot under the fully matched protocol (so \emph{Base} here is the few-shot untrained model, complementing the zero-shot floor anchor of Table~\ref{tab:llama}). The paired \CDRFT{}$-$GRPO gap is positive at nearly every $k$ and widens with $k$ in every domain, the same signature as on Qwen2.5-7B; at $k{=}64$ on ordering-puzzle, GRPO drops below the few-shot base ($37.0 < 41.0$, the coverage collapse of \citet{yue2025rl}) while \CDRFT{} restores it above ($46.0$).

\begin{table}[H]
\centering

\small
\setlength{\tabcolsep}{3.5pt}
\begin{tabular}{lccccccc}
\toprule
\addlinespace[1pt]
& p@1 & p@2 & p@4 & p@8 & p@16 & p@32 & p@64 \\
\addlinespace[1pt]
\midrule
\addlinespace[3pt]
\multicolumn{8}{l}{\textit{GSM8K}} \\
\quad Base & 15.3 & 24.9 & 36.9 & 50.2 & 63.1 & 74.2 & 83.0 \\
\quad GRPO & 16.4 & 26.3 & 38.5 & 51.7 & 64.1 & 74.6 & 83.0 \\
\quad \textbf{\CDRFT} & \textbf{16.8} & \textbf{26.7} & \textbf{38.8} & \textbf{52.0} & \textbf{64.5} & \textbf{75.4} & \textbf{84.2} \\
\addlinespace
\multicolumn{8}{l}{\textit{MATH500}} \\
\quad Base & 7.2 & 11.9 & 18.3 & 25.7 & 33.4 & 40.5 & 47.0 \\
\quad GRPO & 7.4 & 12.3 & 18.9 & 26.8 & 35.0 & 42.5 & 48.4 \\
\quad \textbf{\CDRFT} & \textbf{7.6} & \textbf{12.5} & \textbf{19.2} & \textbf{27.1} & \textbf{35.6} & \textbf{43.8} & \textbf{51.2} \\
\addlinespace
\multicolumn{8}{l}{\textit{code}} \\
\quad Base & 26.5 & 35.2 & 43.4 & 50.8 & 57.3 & 62.7 & 67.5 \\
\quad GRPO & 27.7 & 36.2 & 44.3 & 51.8 & 58.4 & 64.2 & 69.0 \\
\quad \textbf{\CDRFT} & \textbf{28.1} & \textbf{36.7} & \textbf{44.9} & \textbf{52.4} & \textbf{58.9} & \textbf{64.8} & \textbf{69.9} \\
\addlinespace
\multicolumn{8}{l}{\textit{Logic-ordering}} \\
\quad Base & 11.3 & 16.5 & 21.5 & 26.3 & 31.4 & 36.5 & 41.0 \\
\quad GRPO & \textbf{11.7} & \textbf{16.8} & 21.6 & 26.1 & 30.4 & 34.1 & 37.0 \\
\quad \textbf{\CDRFT} & 11.0 & 16.4 & \textbf{21.9} & \textbf{27.4} & \textbf{33.2} & \textbf{39.5} & \textbf{46.0} \\
\addlinespace
\multicolumn{8}{l}{\textit{Logic-graph}} \\
\quad Base & 27.6 & 35.9 & 42.4 & 47.6 & 52.4 & 56.5 & 59.0 \\
\quad GRPO & 29.4 & \textbf{38.4} & \textbf{45.8} & 52.2 & 58.4 & 64.4 & 70.0 \\
\quad \textbf{\CDRFT} & \textbf{29.6} & 38.2 & 45.6 & \textbf{52.3} & \textbf{58.8} & \textbf{65.2} & \textbf{71.0} \\
\bottomrule
\end{tabular}
\caption{Llama-3.2-3B full per-benchmark $\passk$ ladder at the training rollout temperature; all three arms (Base $/$ GRPO $/$ \CDRFT{}) evaluated few-shot under the matched protocol. Bold marks the best of the three. GSM8K is held out from the training mixture of both models. All entries are percentages.}
\label{tab:llama_ladder}

\end{table}

\section{Computational Overhead of the Control Regularizer}\label{sec:overhead}

Table~\ref{tab:overhead} measures what the control regularizer costs per step. The comparison is fully controlled: the same RTX PRO 6000 GPU, the same $\beta{=}0$ backbone, the same generation length, and the same first 1000 steps, with the finite-difference regularizer as the only variable toggled on or off.

\begin{table}[H]
\centering

\small
\begin{tabular}{lcc}
\toprule
\addlinespace[1pt]
Method & Mean step time & Cumulative \\
\addlinespace[1pt]
\midrule
\addlinespace[3pt]
GRPO ($\beta{=}0$) & 30.5 s & 8.5 h \\
\textbf{CD-RFT ($\beta{=}0$)} & 32.9 s & 9.1 h \\
\addlinespace[1pt]
\midrule
\addlinespace[3pt]
Overhead & \textbf{+7.9\%} & +0.6 h \\
\bottomrule
\end{tabular}
\caption{Per-step cost of the control regularizer (Qwen2.5-7B, single RTX PRO 6000, $\beta{=}0$, first 1000 steps, only the regularizer toggled). This is the worst case: at $\beta{=}0$ the inner projection is active at essentially every step.}
\label{tab:overhead}

\end{table}

The measurement is the worst case for our method by construction. At $\beta{=}0$ nothing else holds the concentration down, so the inner projection engages at essentially every step and pays its full cost; under the $\beta{=}0.01$ backbone the KL penalty already suppresses $\Bshared$, the projection fires on only 16 of 1500 steps, and the overhead is correspondingly smaller. We report only the fully-engaged configuration.

Two details fix the reading. The reported value is the mean step time, and both arms include a few steps slowed by node contention, so on the median step the gap is somewhat wider ($22.8\rightarrow25.5$ s, $+11.8\%$). Wall-clock time is also not comparable across GPUs, because among the four main arms the \CDRFT{} models were trained on RTX PRO 6000 and the GRPO models on A100, and since the former is the faster GPU a direct wall-clock comparison would favour the proposed method. The overhead is therefore quoted only from the same-GPU, same-backbone measurement above, which isolates the regularizer itself. In either case the conclusion is the one that Remark~\ref{app:propC7} anticipates, that rollout generation dominates the step, the central-difference probe is a small addition on top of it, and the regularizer remains a low-overhead component.

\end{document}